\documentclass{article}

\usepackage[nonatbib,preprint]{neurips_2023}

\usepackage[utf8]{inputenc} 
\usepackage[T1]{fontenc}    
\usepackage{hyperref}       
\usepackage{url}            
\usepackage{booktabs}       
\usepackage{amsfonts}       
\usepackage{nicefrac}       
\usepackage{microtype}      
\usepackage{xcolor}         
\usepackage{graphicx}
\usepackage{subfigure}
\usepackage{amsmath}
\usepackage{amssymb}
\usepackage{amsthm}
\usepackage{mathrsfs}
\usepackage{algorithm}
\usepackage{algpseudocode}
\usepackage{verbatim}
\usepackage{cite}
\usepackage{hyperref}
\usepackage{enumitem}
\usepackage{tabularx}

\newtheorem{assumption}{\bf Assumption}
\newtheorem{theorem}{\bf Theorem}
\newtheorem{lemma}{\bf Lemma}
\newtheorem{remark}{\bf Remark}
\newtheorem{definition}{\bf Definition}
\newtheorem{corollary}{\bf Corollary}
\newtheorem{proposition}{\bf Proposition}

\title{Understanding Generalization of Federated Learning via Stability: Heterogeneity Matters}

\author{%
  Zhenyu Sun \\
  ECE\\
  Northwestern University\\
  \texttt{zhenyusun2026@u.northwestern.edu} \\
  \And
  Xiaochun Niu  \\
  IEMS\\
  Northwestern University\\
  \texttt{xiaochunniu2024@u.northwestern.edu} \\
  \And
  Ermin Wei \\
  ECE \& IEMS\\
  Northwestern University\\
  \texttt{ermin.wei@northwestern.edu} \\
}

\begin{document}

\maketitle

\begin{abstract}
   Generalization performance is a key metric in evaluating machine learning models when applied to real-world applications. Good generalization indicates the model can predict unseen data correctly when trained under a limited number of data. Federated learning (FL), which has emerged as a popular distributed learning framework, allows multiple devices or clients to train a shared model without violating privacy requirements. While the existing literature has studied extensively the generalization performances of centralized machine learning algorithms, similar analysis in the federated settings is either absent or with very restrictive assumptions on the loss functions. In this paper, we aim to analyze the generalization performances of federated learning by means of algorithmic stability, which measures the change of the output model of an algorithm when perturbing one data point. Three widely-used algorithms are studied, including FedAvg, SCAFFOLD, and FedProx, under convex and non-convex loss functions. Our analysis shows that the generalization performances of models trained by these three algorithms are closely related to the heterogeneity of clients' datasets as well as the convergence behaviors of the algorithms. Particularly, in the i.i.d. setting, our results recover the classical results of stochastic gradient descent (SGD).
\end{abstract}

\section{Introduction}

Federated learning (FL) has recently emerged as an important paradigm for distributed learning in large-scaled networks \cite{FLsurvey}. Unlike the traditional centralized learning, where a model is trained under a large dataset stored at the server \cite{CL1,CL2,CL3}, in federated learning the server hands over computation tasks to the clients, which in turn perform learning algorithms on their local data. After training locally, each client reports its updated model back to the server for model aggregation. The server then aggregates all clients' models to generate a new one that serves as the initialization for the next round of clients' local training. This process is repeated with periodic communication. This local-training framework ensures privacy-preserving and communication-efficient characteristics for federated learning in the sense that no data are transmitting to the server \cite{FLsurvey}.

FedAvg \cite{FedAvg}, the first proposed algorithm satisfying federated learning paradigm, implements stochastic gradient descent (SGD) to update local models, which is simple to implement. However, FedAvg suffers from slow speed of convergence when local data are highly heterogeneous because the local training steps drive the local models away from the global optimal model and towards the local optimal solution \cite{FedAvg,ConFedAvg}. This is called client-drift. To mitigate client-drift causing by data heterogeneity, two promising algorithms are proposed. FedProx \cite{FedProx} uses a proximal method such that the trained local model stays relatively close to the global model. Nevertheless, each client has to solve a proximal point optimization problem during each round which could be computationally expensive. Alternatively, SCAFFOLD \cite{SCAFFOLD} tries to correct for client-drift based on variance reduction. It is proved that SCAFFOLD outperforms FedAvg when the heterogeneity level of data is large and enjoys a faster convergence speed \cite{SCAFFOLD}. Besides, there are several other algorithms proposed later focusing on dealing with client-drift while improving convergence performances \cite{FedAC,FCO}. 

Most existing experimental and theoretical results of FL emphasize on convergence to empirical optimal solutions based on training datasets \cite{FedProx,SCAFFOLD,FedBE,ConFedProx} and often ignore their generalization properties. Generalization of FL 
 is important, as it measures the performance of trained models on unseen data by evaluating its testing error. There are only a few existing works studying the generalization properties.  Generalization bounds are provided for FL \cite{AFL,FedGDA-GT,ICLR23}, which ignores the algorithm choices. These works also require some restrictive assumptions, e.g., binary loss \cite{AFL,FedGDA-GT}, Bernstein condition \cite{ICLR23}.  
In \cite{MAML21,PFLgen}, generalization bounds for meta-learning and federated learning are established respectively, when losses are strongly convex and bounded. However, in many practical scenarios strong convexity does not hold and the loss function may be unbounded. We also note that bounds in \cite{MAML21,PFLgen} are based on uniform stability, which uses a supremum over all single point perturbations. These tend to be overly conservative compared to an alternative stability notion, on-average stability, which takes expectation instead of supremum. Moreover, for the above-mentioned works, the connection between data heterogeneity and generalization performances is not explicitly characterized. Therefore, in this paper, we use on-average stability analysis to obtain generalization bounds that clearly illustrate dependence of data heterogeneity as well as algorithm convergence speed of three widely-used algorithms: FedAvg, SCAFFOLD and FedProx. Our bounds are established under general convex and non-convex losses, which can be unbounded.

\subsection{Related work}

\paragraph{Convergence of federated learning algorithms.} Many recent studies are devoted to federated learning problems due to the increasing volume of data among heterogeneous clients and concerns on privacy leakage and communication cost associated with transmitting users' data for central processing \cite{FedAvg,SCAFFOLD}. FedAvg applies SGD for local updates of clients and suffers from slow convergence performances when the local datasets across clients are highly heterogeneous \cite{ConFedAvg}. To deal with data heterogeneity and improve convergence speed, FedProx adopts proximal methods for local training \cite{FedProx}  and has both convergence guarantees and improved numerical results. SCAFFOLD \cite{SCAFFOLD} borrows the idea from variance reduction methods \cite{SAGA} and shows that convergence rates can be highly improved, compared to FedAvg and FedProx. In \cite{FedNova}, the effects of heterogeneous objectives on solution bias and convergence slowdown are systematically investigated, and FedNova is proposed topreserve fast convergence. FedPD \cite{FedPD} views federated learning from the primal-dual perspective. In \cite{FedLin}, FedLin is aimed to deal with data heterogeneity and system heterogeneity of clients simultaneously. More related works are given therein \cite{FedDANE,FedSplit,FedPA,FedHybrid}.

\paragraph{Generalization of centralized and federated learning.} Generalization of centralized learning has gained attraction of researchers since several decades ago. Uniform convergence is commonly considered to bound the generalization error by means of VC dimension or Rademacher complexity \cite{JMLR10,FML,Yin19,Idan19}. However, uniform convergence sometimes renders the bound too loose to be meaningful \cite{Recht21}. The main reason is that uniform convergence only studies the model class but ignores training algorithms that generates the models. Taking training algorithms into consideration, the generalization bounds might be tighten, since we can directly ignore large amount of models which can never be the output of a specific algorithm. Algorithmic stability is a useful notion that specifically helps to investigate generalization errors by considering dependency on particular algorithms \cite{BE02}. Generalization bounds are built for several stochastic gradient-based methods via stability tools \cite{SGDgen,SGDdata18,SGLDgen}. In terms of federated learning, \cite{AFL} provides uniform convergence bound with rate $\mathcal{O}(1/\sqrt{n})$ for agnostic federated learning problems by Rademacher complexity under binary losses, and $n$ is the number of samples collected by all clients. \cite{ICLR23} studies the case when some clients do not participate during the training phase and establishes bounds with faster rate $\mathcal{O}(1/n)$ under Bernstein condition and bounded losses. And it further requires that clients' distributions are sampled from a meta-distribution, which may be impractical. \cite{MAML21,PFLgen} provide generalization bounds via uniform stability, obtaining rates $\mathcal{O}(1/n)$. Further, \cite{MAML21} requires there is only one-step local update which does not match the common practice of using multiple local updates in federated setting. Note these works all require bounded and strongly convex loss functions. Moreover, none of the above-mentioned works reveals the clear influence of heterogeneous datasets on generalization of federated learning models. In this paper, our theoretical results bridge this gap. Comparison of our results to the existing ones is listed in Table \ref{tab:comparison}.

\begin{table}
    \centering
    \caption{Generalization bounds for federated learning. C, SC, NC denote convex, strongly convex, non-convex, respectively. The last column represents connections of bounds to data heterogeneity.}
    \begin{tabular}{p{5em}p{10.5em}p{9em}p{10em}}
         \toprule
         Reference  &  Loss Function  &  Complexity  &  Distribution Dependence  \\
         \midrule
         \cite{AFL}  &  Bounded, Binary  &  $\mathcal{O}({1}/{\sqrt{n}})$  &  Yes  \\
         \midrule
         \cite{ICLR23}  & Bounded, Bernstein Con  &  $\mathcal{O}({1}/{n})$  & No  \\
         \midrule
         \cite{MAML21,PFLgen} &  Bounded, SC, Smooth  &  $\mathcal{O}({1}/{n})$  &  No   \\
         \midrule
         Ours  &  Unbounded, C, Smooth  &  $\mathcal{O}({1}/{n})$  &  Yes   \\
         \midrule
         Ours  &  Unbounded, NC, Smooth  &  $\mathcal{O}({1}/{n})$  &  Yes   \\
         \bottomrule
    \end{tabular}
    \label{tab:comparison}
\end{table}

\subsection{Our contributions}

We summarize our main contributions as follows: (1) We propose a bound on  generalization error by using  algorithm-dependent on-average stability  in federated settings (see Section \ref{sec_stability}); (2) Based on on-average stability, we provide generalization upper bounds for FedAvg, SCAFFOLD and FedProx respectively with unbounded, convex and non-convex loss functions, which explicitly reveal the effects of data heterogeneity and convergence performances of different algorithms (see Sections \ref{sec_bounds-convex} and \ref{sec_bounds-nonconvex}); (3) In i.i.d. setting with convex loss functions, our bounds match existing results of SGD in the sense that FedAvg reduces to SGD method (see Section \ref{sec_bounds-convex}); (4) Experimental results are provided, demonstrating the trends in our theoretical bounds (see Section \ref{sec_experiments}).

\paragraph{Notations.} We define the $l_2$-norm of finite dimensional vectors as $\Vert \cdot \Vert$. Vectors and scalars for client $i$ are denoted by subscript $i$, e.g., $R_i(\cdot)$. Subscripts, e.g., $t$ or $k$, denote the index of iteration. We let $[n]$ denote the set $\{1,\dots,n\}$ for any positive integer $n$. When taking expectation over some random variable $z$ (which can be multi-dimensional), we denote $\mathbb{E}_{z}[\cdot]$ and we drop $z$ when the context is clear for simplicity.

\section{Problem Formulation}   \label{sec_formulation}

In this paper, we consider the general federated learning problem, where $m$ clients collaboratively minimize the following global population risk formed by
\begin{equation}    \label{eq_global-pop-risk}
    R(\theta) := \sum_{i=1}^m p_i \mathbb{E}_{z \sim P_i}[l(\theta ; z)] 
\end{equation}
where $\theta \in \mathbb{R}^d$ is the parametrized model and $P_i$ is the underlying distribution of the dataset maintained by client $i$. We adopt the standard assumption that $P_i$ and $P_j$ are independent for any $i,j\in [m]$ such that $i\neq j$, motivated by the observation that  local data of clients are commonly unrelated in practical scenarios. We define $z$ as the sample generated by $P_i$, i.e., $z \sim P_i$, $l(\cdot;z)$ as the loss function evaluated at sample $z$, and $p_i$ as some constant scalar that measures the contribution of client $i$'s data to the global objective. We also define the local population risk as $R_i(\theta) := \mathbb{E}_{z \sim P_i} [l(\theta;z)]$.

However, in practice, we are unable to minimize the global population risk directly due to the unknown distributions $P_i$. Thus, one alternative way to get an approximate model is by collecting some empirical sample dataset $\mathcal{S}_i$. More specifically, each local dataset is defined by $\mathcal{S}_i := \{ z_{i,j} \}_{j=1}^{n_i}$, where $z_{i,j}$ is the $j$-th sample of client $i$ and $n_i$ is the number of local samples. Let $\mathcal{S} := \bigcup_{i=1}^m \mathcal{S}_i$ be the dataset with all samples and $n$ be the total amount of samples with $n = \sum_{i=1}^m n_i$. Moreover, we are interested in the {\it balanced case}, i.e., $p_i = n_i/n$, meaning the contribution of each client to the global objective is proportional to the local sample size $n_i$. Thus, we turn to train a model by minimizing the following global empirical risk:
\begin{equation}    \label{eq_global-emp-risk}
    \hat{R}_{\mathcal{S}}(\theta) := \sum_{i=1}^m p_i \hat{R}_{\mathcal{S}_i}(\theta) = \frac{1}{n} \sum_{i=1}^m \sum_{j=1}^{n_i} l(\theta; z_{i,j}),
\end{equation}
where we use the fact that $p_i = {n_i}/{n}$ and $\hat{R}_{\mathcal{S}_i}(\theta)$ is the local empirical risk $
    \hat{R}_{\mathcal{S}_i}(\theta) := \frac{1}{n_i} \sum_{j=1}^{n_i} l(\theta; z_{i,j}) $.
 Here we use superscript notation $\hat R$ to indicate the empirical version of $R$ and will use superscript in a similar fashion for the rest of the paper. Based on the above definitions, we  further define the ground-truth model $\theta^*$ by minimizing the population risk \eqref{eq_global-pop-risk}, that is, 
$
    \theta^* \in \arg\min_{\theta} R(\theta)
$
and correspondingly the best empirically trained model is defined by
$
    \hat{\theta}_{\mathcal{S}} \in \arg\min_{\theta} \hat{R}_{\mathcal{S}}(\theta) $.

Our ultimate goal is to obtain the ground-truth model $\theta^*$, which is impossible due to unknown distributions. What we can do practically is to solve for $\hat{\theta}_{\mathcal{S}}$ by implementing appropriate optimization algorithms such that \eqref{eq_global-emp-risk} is minimized. Then, a natural question is how  we could expect the trained model $\hat{\theta}_{\mathcal{S}}$ to be close to $\theta^*$? Alternatively, we want to test model $\hat{\theta}_{\mathcal{S}}$ on any unseen data such that the testing error is small enough, which means the model $\hat{\theta}_{\mathcal{S}}$ generalizes well on any testing set. 

In general, even given good datasets, exactly obtaining $\hat{\theta}_{\mathcal{S}}$ is still a hard optimization problem. A more reasonable approach is to implement some algorithm $\mathcal{A}$ which outputs a model $\mathcal{A}(\mathcal{S})$, noting the model is a function of the training set $\mathcal{S}$.

\section{Generalization and Stability}  \label{sec_stability}
As stated in the previous section, we now focus on the generalization performance of the output of some algorithm $\mathcal{A}(\mathcal{S})$, given a training dataset $\mathcal{S}$. Mathematically, the generalization error of a model $\mathcal{A}(\mathcal{S})$ is defined by
\begin{align*}    
    \epsilon_{gen} := \mathbb{E}_{\mathcal{S}} \mathbb{E}_{\mathcal{A}} [R(\mathcal{A}(\mathcal{S})) - \hat{R}_{\mathcal{S}}(\mathcal{A}(\mathcal{S}))],
\end{align*}
where the expectation is taken over $\mathcal{S}$ to model the random sampling of data and over $\mathcal{A}$ to allow the usage of  randomized algorithms. For instance, if stochastic gradient is used in an algorithm then the expectation over $\mathcal{A}$ is average over different samples used to compute the stochastic gradients. A smaller $\epsilon_{gen}$ implies the model $\mathcal{A}(\mathcal{S})$ has a better generalization performance on  testing datasets. 

Generally speaking, it is hard to characterize the generalization error due to the implicit dependency of the model and the training dataset. In this paper, we apply the notion of algorithmic stability to provide an upper bound on the generalization error. In particular, we formally define the on-average stability in the context of federated learning. To do this, we first introduce the definition of neighboring datasets. 

\begin{definition}  \label{def_neighboring-data}
    Given a global dataset $\mathcal{S} = \bigcup_{l=1}^{m} \mathcal{S}_l$, where $\mathcal{S}_l$ is the local dataset of the $l$-th client with $\mathcal{S}_l = \{ z_{l,1},\dots, z_{l, n_l} \}, \forall l \in [m]$, another global dataset is said to be neighboring to $\mathcal{S}$ for client $i$, denoted by $\mathcal{S}^{(i)}$, if $\mathcal{S}^{(i)} := \bigcup_{l \ne i} \mathcal{S}_l \cup \mathcal{S}'_i$, where $\mathcal{S}'_i = \{ z_{i,1}, \dots, z_{i, j-1}, z'_{i,j}, z_{i, j+1}, \dots, z_{i, n_i} \}$ with $z'_{i,j} \sim P_i$, for some $j \in [n_i]$. And we call $z'_{i,j}$ the perturbed sample in $\mathcal{S}^{(i)}$.
\end{definition}

In other words, $\mathcal{S }$ and $\mathcal{S}^(i)$ are neighboring datasets if they only differ by one data point in $\mathcal{S}_i$ and both are sampled from the same local distribution. Then, we have the following definition of on-average stability for federated learning algorithms, which is established based on on-average stability for centralized learning \cite{JMLR10}.

\begin{definition}[on-average stability for federated learning]    \label{def_stability}
    A federated learning algorithm $\mathcal{A}$ is said to have $\epsilon$-on-average stability if given any two neighboring datasets $\mathcal{S}$ and $\mathcal{S}^{(i)}$, then 
    \begin{equation*}
        \max_{j \in [n_i]} \mathbb{E}_{\mathcal{A, S}, z'_{i,j}}| l(\mathcal{A}(\mathcal{S}); z'_{i,j}) - l(\mathcal{A}(\mathcal{S}^{(i)}); z'_{i,j}) | \le \epsilon , \quad\forall i \in [m] ,
    \end{equation*}
    where $z'_{i,j}$ is the perturbed sample in $\mathcal{S}^{(i)}$.
\end{definition}

On-average stability basically means any perturbation of samples across all clients cannot lead to a big change of the model trained by the algorithm in expectation. The next theorem shows that on-average stability can be used to bound the generalization error of the model. The proof is given in Appendix \ref{apx_proof-gen-stability}.

\begin{theorem} \label{thm_stability-gen}
    Suppose a federated learning algorithm $\mathcal{A}$ is $\epsilon$-on-averagely stable. Then,
    \begin{align*}
        \epsilon_{gen} \le \mathbb{E}_{\mathcal{A, S}} \left[ \left| R(\mathcal{A(S)}) - \hat{R}_{\mathcal{S}}(\mathcal{A(S)}) \right| \right] \le  \epsilon .
    \end{align*}
\end{theorem}

Therefore, it suffices to characterize the on-average stability of a federated learning algorithm to bound the generalization error of the model. Theorem \ref{thm_stability-gen} extends the classical connection of on-average stability and generalization \cite{JMLR10}, where no heterogeneity characteristic of datasets is considered. Based on Definition \ref{def_stability}, we show that when the perturbation of a sample for any local agent has a small influence on algorithm output (i.e., a small $\epsilon$), the generalization error is also small (i.e., $\epsilon_{gen}$ is small). This relationship always holds given any clients' local data distributions. Then, in the following, we focus on analyzing the stability of different federated learning algorithms and applying their stability results to the measure of generalization.

\section{Summary of Federated Learning Algorithms}  \label{sec_algorithms}
In this section, we briefly summarize three widely-used federated learning algorithms: FedAvg, SCAFFOLD, and FedProx, based on which the generalization bounds would be provided. To simplify the analysis, we assume that there is no partial participation among the clients, but our analysis can be extended to partial participation scenarios as well.

Any federated algorithms can be decomposed into two stages: local updating and model aggregation. At the beginning of each communication round (time index $t$), the server maintains a global model $\theta_t$, which is sent to all clients serving as an initial model of local updating. All clients update their local models $\theta^i_{t+1}$ based on their own datasets in parallel. Then,  a model aggregation for the start of next round, i.e.,  $\theta_{t+1}=\sum_{i=1}^{m} p_i \theta^i_{t+1}$. The three methods only differ in their local updating procedures and the detailed descriptions of algorithms can be found in Appendix \ref{apx_algorithms}.

For FedAvg and SCAFFOLD, in the $t$-th communication round, we assume that there are $K_i$ local updates and denote by $\theta_{i,k}$ and $g_i(\cdot)$  the local model at local iteration $k$ and the sampled gradient of agent $i$. For FedProx, we let $\theta^i_{t+1}$ be the model of client $i$ after local training at round $t$. Then, for client $i$, the local updates at iteration $k$ are described as follows.
\begin{itemize}[leftmargin=*]
    \item FedAvg: Let $\alpha_{i,k}$ be the constant (or diminishing) local stepsize, agent $i$'s local update is
    \begin{eqnarray}    \label{eq_FedAvg-update}
        \theta_{i,k+1} = \theta_{i,k} - \alpha_{i,k} g_i(\theta_{i,k}), \quad \forall k=0,\dots, K_i-1.
    \end{eqnarray}
    \item SCAFFOLD: Let $\alpha_{i,k}$ be the constant (or diminishing) local stepsize, agent $i$'s local update is
    \begin{eqnarray}    \label{eq_SCAFFOLD-update}
        \theta_{i,k+1} = \theta_{i,k} - \alpha_{i,k}(g_i(\theta_{i,k}) - g_i(\theta_t) + g(\theta_t)) , \quad\forall k=0,\dots, K_i-1 ,
    \end{eqnarray}
    where $g(\theta_t) = \sum_{i=1}^m p_i g_i(\theta_t)$ is the aggregation of all locally sampled gradients.
    \item FedProx: Let $\eta_i$ be a constant parameter for the proximal term, agent $i$'s local update is
    \begin{eqnarray}    \label{eq_FedProx-update}
        \theta^i_{t+1} = \arg\min_{\theta} ~ \hat{R}_{\mathcal{S}_i}(\theta) + \frac{1}{2\eta_i} \Vert \theta - \theta_t \Vert^2.
    \end{eqnarray}
\end{itemize}

\section{Main Results}  \label{sec_bounds}

In this section, we provide bounds on the generalization errors for FedAvg, SCAFFOLD, and FedProx mentioned in the last section by studying the on-average stability in Definition \ref{def_stability}. We allow the loss functions to be unbounded from above, which can be convex or nonconvex. Intuitively, different local distributions affect the global population risk \eqref{eq_global-pop-risk} and hence may affect the model generalization as well. To measure the heterogeneity of client $i$'s data, we use $D_i$ to denote the total variation of $P_i$ and $P$, i.e., $D_i := d_{TV}(P_i, P)$ with $P=\sum_{i=1}^m p_i P_i$. Moreover, we define $D_{max} := \max_{i \in [m]} D_i$ to measure the furthest distance between the global distribution and any local distribution\footnote{Our bounds can be derived under KL divergence as well. However, bounds involving total variation are tighter.}. A larger value of $D_{max}$ means greater heterogeneity among the clients. Throughout the analysis we require the following assumptions.

\begin{assumption}  \label{assump_Lip-continuous}
    The loss function $l(\cdot, z)$ is $L$-Lipschitz continous for any sample $z$, that is,
    $
        |l(\theta; z) - l(\theta'; z) | \le L \Vert \theta - \theta' \Vert$,  for any $z, \theta, \theta'$.
\end{assumption}

\begin{assumption}  \label{assump_bounded-grad-var}
    Assume that for any $\theta$, $i\in [m]$, and $z_{i,j} \sim P_i$,
    $
        \mathbb{E} \left[ \Vert \nabla l(\theta; z_{i,j}) - \nabla R_{\mathcal{S}_i}(\theta) \Vert^2 \right] \le \sigma^2$, for any $j \in [n_i]$.
\end{assumption}

\begin{assumption}  \label{assump_Lip-smooth}
    The loss function $l(\cdot, z)$ is $\beta$-smooth for any $z$, that is, 
    $
        \Vert \nabla l(\theta; z) - \nabla l(\theta';z) \Vert \le \beta \Vert \theta - \theta' \Vert$,  for any $z, \theta, \theta'$.
\end{assumption}

Assumption \ref{assump_Lip-continuous} is standard in literature \cite{SGDgen} to establish the connection between model perturbation with stability\footnote{We note that we only need  Assumption \ref{assump_Lip-continuous} to hold  for all iterates generated by the algorithms, which is trivially satisfied, because the methods are convergent and the  iterates are from a compact set.}. Assumptions \ref{assump_bounded-grad-var} and \ref{assump_Lip-smooth} serve in our analysis to capture the heterogeneity of different datasets as well as the influence of convergence performances of different algorithms. Detailed proofs of this section are in Appendices \ref{apx_proof-convex} and \ref{apx_proof-nonconvex}.

\subsection{Convex loss functions}  \label{sec_bounds-convex}

We first study the case when the loss function is convex with respect to the model parameter.

\begin{assumption}  \label{assump_convexity}
    The loss function $l(\cdot, z)$ is convex for any $z$.
\end{assumption}

For each of the three algorithms, FedAvg, SCAFFOLD and FedProx with local updates \eqref{eq_FedAvg-update}, \eqref{eq_SCAFFOLD-update}, \eqref{eq_FedProx-update},  we apply the method to two neighboring training datasets, i.e.,  only one data point of one agent is different. We then analyze and bound the difference between the resulting models by data heterogeneity and algorithm performances.

\begin{theorem} \label{thm_gen-convex}
    Under Assumptions \ref{assump_Lip-continuous}-\ref{assump_convexity}, denote $\{\theta_t\}_{t=0}^T$ and $\{\theta'_t\}_{t=0}^T$ as the trajectories of the server's models induced by neighboring datasets $\mathcal{S}$ and $\mathcal{S}^{(i)}$, respectively. Furthermore, suppose the same initialization, i.e., $\theta_0 = \theta'_0$. Then, we have the following bounds on resulting models.
    
    For FedAvg,
    \begin{align*}
        \mathbb{E}\Vert \theta_T - \theta'_T \Vert \le \frac{2}{n} \sum_{t=0}^{T-1} \tilde{\alpha}_{i,t}(1 + \beta \tilde{\alpha}_{i,t}) \Big( 2L D_i + \mathbb{E}\Vert \nabla R(\theta_t) \Vert + \sigma \Big) .
    \end{align*}

    For SCAFFOLD,
    \begin{align*}
        \mathbb{E}\Vert \theta_T - \theta'_T \Vert \le \frac{2}{n} \sum_{t=0}^{T-1} \exp{\Big(2\beta \sum_{l=t+1}^{T-1} \hat{\alpha}_l \Big)} \Big( 2L D_i \gamma^1_{t} + \gamma^2_{t} \mathbb{E}\Vert \nabla R(\theta_{t}) \Vert  + \sigma \gamma^2_{t} \Big)
    \end{align*}

    For FedProx,
    \begin{align*}
        \mathbb{E}\Vert \theta_T - \theta'_T \Vert \le \frac{2}{n} \sum_{t=0}^{T-1} \eta_{i}(1 + \beta \eta_{i}) \Big( 2L D_i + \mathbb{E}\Vert \nabla R(\theta_t) \Vert + \sigma \Big) ,
    \end{align*}
    where 
        $\gamma^1_t := 2 \tilde{\alpha}_{i,t} + \hat{\alpha}_t$ and $\gamma^2_t := \gamma^1_t + \beta \tilde{\alpha}^2_{i,t}$
    with $\tilde{\alpha}_{i,t} := \sum_{k=0}^{K_i - 1} \alpha_{i,k}$,  $\hat{\alpha}_t := \sum_{j=1}^m p_j \tilde{\alpha}_{j,t}$, and $\sum_{l=T}^{T-1} \hat{\alpha}_l = 0, \forall \hat{\alpha}_l$. The expectations are taken with respect to $\mathcal{S}$ and $\mathcal{S}^{(i)}$ jointly as well as the randomness of algorithms. 
\end{theorem} 

Next we discuss the implications of Theorem \ref{thm_gen-convex}. Firstly, the model differences of the three algorithms  all linearly increase in $D_i$. Recall that $D_i$  is the total variation of data distribution of client $i$ and the global distribution $P$, measuring the heterogeneity level of client $i$'s data. This dependency is due to the fact that we only perturb one data point of client $i$ while keeping the others the same and hence only client $i$'s distribution comes into the bound. As $D_i$ increases, perturbing one data point at client $i$'s dataset corresponds to a bigger change in the overall dataset and therefore the distance between the two models increases.
    
Secondly, the sequence of global gradients evaluated along the trajectories, i.e., $ \{ \mathbb{E} \Vert \nabla R(\theta_t) \Vert \}_{t=0}^{T-1}$, influences the bounds of model differences. Note that this effect is essentially determined by the convergence performances of algorithms, in the sense that $ \{ \mathbb{E} \Vert \nabla R(\theta_t) \Vert \}_{t=0}^{T-1}$ captures how fast $\{ \theta_t \}_{t=0}^{T-1}$ approaches to the optimal solution $\theta^*$. Faster converging methods correspond to smaller $ \{ \mathbb{E} \Vert \nabla R(\theta_t) \Vert \}_{t=0}^{T-1}$ terms.

Thirdly, the bounds are also proportional to the sampling variance $\sigma^2$ of gradients. A small  $\sigma$ indicates the sampled gradient is accurate and is close to
the true gradient $\nabla R(\cdot)$. In particular, when $\sigma = 0$, each client is able to compute $\nabla R_i(\cdot)$ exactly, in which case the bounds are only related to data heterogeneity and algorithm convergence performances.

Finally, all three bounds depend on stepsizes chosen during the local training process. Different choices of stepsizes result in different convergence rates of algorithms. From the above results, larger stepsizes may make algorithms less "stable", i.e., $\Vert \theta_T - \theta'_T \Vert$ becomes bigger, as any difference caused by the perturbed data is magnified by the stepsize.

As we discussed above, the summation of  $\mathbb{E}\Vert \nabla R(\theta_t) \Vert$ is related to  the convergence speed of the algorithm. In the following theorem, we focus on characterizing these terms as a function of the number of iterations.

\begin{theorem} \label{thm_convergence}
    Under Assumptions \ref{assump_Lip-continuous}-\ref{assump_Lip-smooth}, suppose $K_i = K$ and $\alpha_{i,k} \le {1}/(24\beta K)$  for any $i=1,\dots,m$. Then, for FedAvg, we have
    \begin{align*}
        \sum_{t=0}^{T-1} \tilde{\alpha}_{i,t} (1 + \beta \tilde{\alpha}_{i,t}) \mathbb{E}\Vert \nabla R(\theta_t) \Vert = \mathcal{O} \Big( \big( \frac{\Delta_0}{K m} \big)^{\frac{1}{4}} T^{\frac{3}{4}} + \big( \Delta_0^2 \sum_{i=1}^m p_i D_i^2 \big)^{\frac{1}{6}} T^{\frac{2}{3}} + \sqrt{\Delta_0} T^{\frac{1}{2}} \Big) .
    \end{align*}
    For SCAFFOLD, if we further set $\alpha_{i,k} \le 1/[24\beta K (t+1)]$, then
    \begin{align*}
        \sum_{t=0}^{T-1} \exp{\Big(2\beta \sum_{l=t+1}^{T-1} \hat{\alpha}_l \Big)} \gamma^2_{t} \mathbb{E}\Vert \nabla R(\theta_{t}) \Vert = \mathcal{O}\Big( \big(\frac{\Delta_0}{Km}\big)^{\frac{1}{4}} T^{\frac{5}{6}} + \sqrt{\Delta_0}T^{\frac{7}{12}} \Big) .
    \end{align*}
    For FedProx, we have
    \begin{align*}
        \sum_{t=0}^{T-1} \eta_i (1 + \beta \eta_i) \mathbb{E}\Vert \nabla R(\theta_t) \Vert = \mathcal{O} \Big( \big( \Delta_0 \sum_{i=1}^m p_i D_i^2 \big)^{\frac{1}{2}} T^{\frac{3}{4}} + \sqrt{\Delta_0}T^{\frac{1}{2}} \Big) .
    \end{align*}
    We denote by $\Delta_0 = \mathbb{E}[R(\theta_0) - R(\theta^*)]$ to represent the distance of initial population risk  based on  $\theta_0$ to the ground-truth $\theta^*$.
\end{theorem}

Theorem \ref{thm_convergence} bounds the global gradients $\Vert \nabla R(\theta_t) \Vert$ along the trajectories of server's outputs. This theorem  holds for both convex and non-convex settings. Under suitable selections of stepsizes, Theorem \ref{thm_convergence}  implies that the global gradient $\mathbb{E} \Vert \nabla R(\theta_t) \Vert$ converges to zero. This is consistent with convergence results in the optimization perspective\cite{FedProx,SCAFFOLD}. To see this, when dividing the preceding bounds by $T$, the right hand side converge to zero in polynomial times and hence $\mathbb{E} \Vert \nabla R(\theta_t) \Vert$  must converge to zero. Moreover, these bounds increases with $\Delta_0$, which measures the distance of initial model to the optimal one. Thus, starting at a model closer  to the optimal solution requires less number of iterations to approximate accurately $\theta^*$.

By combining Theorems \ref{thm_stability-gen}-\ref{thm_convergence}, we establish the generalization bounds for three algorithms, respectively. We also define $\tilde{D} := \sum_{i=1}^{m} p_i D_i^2$ and $\Delta_0 = \mathbb{E}[R(\theta_0) - R(\theta^*)]$.

\begin{corollary}   \label{coro_gen-convex}
    Suppose Assumptions \ref{assump_Lip-continuous}-\ref{assump_convexity} hold and the selection of stepsizes are the same as Theorem \ref{thm_convergence}. Then, we have the following generalization bounds: For FedAvg, 
    \begin{align*}
        \epsilon_{gen} \le \mathcal{O}\Big( \frac{T}{n} D_{max} \Big) + \mathcal{O} \Big( \big( \frac{\Delta_0}{K m} \big)^{\frac{1}{4}} \frac{T^{\frac{3}{4}}}{n} + \big( \Delta_0^2 \tilde{D} \big)^{\frac{1}{6}} \frac{T^{\frac{2}{3}}}{n} +  \sqrt{\Delta_0} \frac{T^{\frac{1}{2}}}{n} \Big) + \mathcal{O}\Big( \frac{\sigma T}{n} \Big).
    \end{align*}
    For SCAFFOLD, 
    \begin{align*}
        \epsilon_{gen} \le \mathcal{O}\Big( \frac{T^{\frac{1}{12}} \log{T}}{n} D_{max} \Big) + \mathcal{O}\Big( \big( \frac{\Delta_0}{K m} \big)^{\frac{1}{4}} \frac{T^{\frac{5}{6}}}{n} +  \sqrt{\Delta_0} \frac{T^{\frac{7}{12}}}{n} \Big) + \mathcal{O}\Big( \frac{T^{\frac{1}{12}}(1 + \log{T})}{n}\sigma \Big).
    \end{align*}
    For FedProx, 
    \begin{align*}
        \epsilon_{gen} \le \mathcal{O}\Big( \frac{T}{n} D_{max} \Big) + \mathcal{O}\Big( \big( \Delta_0 \tilde{D} \big)^{\frac{1}{2}} \frac{T^{\frac{3}{4}}}{n} + \sqrt{\Delta_0} \frac{T^{\frac{1}{2}}}{n} \Big) + \mathcal{O}\Big( \frac{T}{n}\sigma \Big).
    \end{align*}
\end{corollary}

As indicated in Theorem \ref{thm_gen-convex},  the generalization bound for each algorithm can be separated into three terms corresponding to the three $\mathcal{O}(\cdot)$ terms: heterogeneity level (first), convergence performance (second), sampling variance (third). Note $D_{max}$ in the first term measures data heterogeneity among all agents. A smaller $D_{max}$ indicates clients have more similar datasets, which has a positive effect on the generalization of trained models.  Moreover, generalization bounds above scale inversely with $n$, which is the total sample size. This implies increasing the number of samples gives a better generalization performance. Note that fixing $T$, the rate is with the order of $\mathcal{O}(1/n)$, which is the same as results in the centralized setting \cite{SGDgen,SGDdata18,Yu18}.


Furthermore, when assuming all clients maintain i.i.d. data and applying the Lipschitz continuity condition to bound the gradient, i.e., $\Vert \nabla R(\cdot) \Vert \le L$, we have the following bounds under suitable choices of stepsizes.

\begin{corollary}   \label{coro_gen-iid}
    We suppose Assumptions \ref{assump_Lip-continuous}-\ref{assump_convexity} hold and all clients have i.i.d. datasets, that is, $P_i = P_j$, for any $i,j \in [m]$.
    Then for FedAvg and FedProx, if stepsizes are chosen to be constant, we have $\epsilon_{gen} \le \mathcal{O} ( {L(L + \sigma)}T/{n})$. For SCAFFOLD, if stepsizes are chosen with the order of $\mathcal{O}(c/(2\beta t))$, we have $\epsilon_{gen} \le \mathcal{O}( {L(L + \sigma)} T^{c} \log{T}/{n} )$.
\end{corollary}

From Corollary \ref{coro_gen-iid}, we observe that the heterogeneity term disappears because $D_{max} = 0$ in i.i.d. settings. If $\sigma$ is relatively small compared to $L$, for FedAvg, our result is aligned with the bound for SGD \cite{SGDgen}. The reason is that for i.i.d. case $R_i(\cdot) = R(\cdot)$ and thus the server's model essentially performs  SGD (in expectation). For SCAFFOLD, the update reduces to SAGA \cite{SAGA}. Therefore, the generalization bound for SAGA is implied by Corollary \ref{coro_gen-iid}.

If we set $D_{max}=0$ and choose comparable stepsizes, the bounds of Corollary \ref{coro_gen-convex} are tighter than those of Corollary \ref{coro_gen-iid}. The main reason is that using Lipschitz constant $L$ to bound the gradient is usually too loose and the algorithm performances are highly ignored, which, however, should be carefully considered in analysis. In particular, considering FedAvg with $m=1$ and $K=1$, which is then equivalent to classical SGD method, our result is better than the result provided in \cite{SGDgen} when stepsizes are constants, where the order of $T$ reduces from $\mathcal{O}(T)$ to $\mathcal{O}(T^{5/6})$.

\subsection{Non-convex losses}  \label{sec_bounds-nonconvex}

In many practical scenarios, the loss functions are non-convex (e.g., neural networks). Therefore, we provide generalization bounds for non-convex losses in this subsection.

\begin{theorem} \label{thm_gen-nonconvex}
    Under Assumptions \ref{assump_Lip-continuous}-\ref{assump_Lip-smooth}, suppose $K_i = K$ and $\alpha_{i,k} \le \frac{1}{24\beta K (t+1)}$ for FedAvg and SCAFFOLD. Then for FedAvg we have
    \begin{align*}    
        \epsilon_{gen} \le \mathcal{O}\Big( \frac{T^{\frac{1}{24}} \log{T}}{n} (D_{max}+\sigma) \Big) + \mathcal{O} \Big( \big( \frac{\Delta_0}{K m} \big)^{\frac{1}{4}} \frac{T^{\frac{5}{6}}}{n} + \big( \Delta_0^2 \tilde{D} \big)^{\frac{1}{6}} \frac{T^{\frac{3}{4}}}{n} + \sqrt{\Delta_0} \frac{T^{\frac{7}{12}}}{n} \Big) .
    \end{align*}
    For SCAFFOLD,
    \begin{align*}    
        \epsilon_{gen} \le \mathcal{O}\Big( \frac{T^{\frac{1}{8}} \log{T}}{n}  D_{max} \Big) + \mathcal{O} \Big( \big( \frac{\Delta_0}{K m} \Big)^{\frac{1}{4}} \frac{T^{\frac{7}{8}}}{n} + \sqrt{\Delta_0} \frac{T^{\frac{5}{8}}}{n} \big) + \mathcal{O} \Big( \frac{T^{\frac{1}{8}} (\log{T} + 1)}{n} \sigma \Big).
    \end{align*}
    For FedProx, if the eigenvalues of $\nabla^2 R_i(\theta)$ are lower bounded and $\eta_i$ is chosen small enough and diminishing with order $\mathcal{O}(c/t)$, then
    \begin{eqnarray}    \label{eq_FedProx-nonconvex}
        \epsilon_{gen} \le \tilde{\mathcal{O}}\Big( \frac{T^c}{n} D_{max} \Big) + \mathcal{O}\Big( \big( \Delta_0 \tilde{D} \big)^{\frac{1}{2}} \frac{T^{\frac{3}{4}+c}}{n} + \sqrt{\Delta_0} \frac{T^{\frac{1}{2}+c}}{n} \Big) + \tilde{\mathcal{O}}\Big( \frac{T^c}{n}\sigma \Big) .
    \end{eqnarray}
\end{theorem}

In Theorem \ref{thm_gen-nonconvex}, the bounds are similar to those of convex cases, i.e., data heterogeneity, algorithm convergence and sampled variance jointly affect the generalization error of the models. In addition, we remark that in practice $T$ is usually characterized by a function of $n$ and $m$. Then in this sense, the generalization bounds can be further simplified in terms of the total sample size $n$ and the number of clients $m$. For example, considering \eqref{eq_FedProx-nonconvex} with one full pass local training, i.e., $T = \mathcal{O}(n/m)$, we obtain $\epsilon_{gen} \le \mathcal{O}(m^{-c}n^{-(1-c)} + m^{-3/4-c}n^{-(1/4-c)})$, meaning the generalization error diminishes as the number of clients participating in the learning process increases.

However, we remark that only upper bounds on generalization errors are provided, and some constants are ignored in our $\mathcal{O}(\cdot)$ notation. In reality, these ignored constants (e.g. stepsizes) could largely affect algorithms' generalization performances. Thus, our bounds might not be tight enough to explain accurately the performances of algorithms. In addition, for different algorithms, the selection of stepsizes is usually a tricky task, meaning optimal stepsizes for different algorithms are chosen during the training process. This implies, in general, it is hard to compare generalization errors among different algorithms by directly analyzing our bounds. Instead, the main insight shown by our results is that explicit dependency of data heterogeneity to generalization is clearly characterized through total variation among local distributions. This is a first step towards this direction, as existing  literature \cite{MAML21,PFLgen,ICLR23} fails to characterize the connection of data heterogeneity to generalization bounds.

\section{Experiments} \label{sec_experiments}

In this section, we numerically evaluate the generalization errors of models trained by FedAvg, SCAFFOLD and FedProx under non-convex loss functions, given different heterogeneity levels of clients' datasets.

{\bf Experimental Setups.} We investigate classification problems using the MNIST dataset \cite{dataset}. Each client maintains a three-layer neural network comprising two convolutional layers and a fully connected layer. We focus on a federated learning system involving $10$ clients. The training is based on Personalized Federated Platform \cite{platform}.

Next, we elaborate on how we construct different clients' datasets with different heterogeneity levels. We introduce various levels of heterogeneity among the clients' local data distributions to examine the impact of heterogeneity on the generalization performance. In the case of extreme heterogeneity, labeled ``fully non-i.i.d.'' scenario, each client has two specific labels out of the ten available MNIST. In the ``i.i.d.'' scenario, the local datasets are uniformly mixed with all ten labels. We also consider the intermediate scenarios labeled ``$\rho$ non-i.i.d.'', where a fraction $\rho$ of data follows the ``fully non-i.i.d.'' assignment, while the remaining fraction $1-\rho$ adheres to the ``i.i.d.'' assignment. We call $\rho$ the heterogeneity level of local data distributions and run the experiments for  $\rho=0, 0.2,0.5,0.8, 1$ cases ($5$ in total), where  $\rho=0$ and $\rho=1$ are the ``i.i.d.'' and ``fully non-i.i.d.'' cases, respectively.

In different settings, we start the algorithms from the same initial value with the same training loss. As the training goes on, the training losses decrease. We compare trained models under different levels of training losses. To quantify the generalization errors, we use the absolute difference between the training and testing losses, i.e., $| R(\mathcal{A}(S)) - \hat{R}_{\mathcal{S}}(\mathcal{A}(\mathcal{S})) |$. We terminate the algorithms when either the training loss reaches a desirable level or the number of training steps achieves $T=1000$.
\begin{figure}[h]
    \centering
    \subfigure[FedAvg]{\includegraphics[width=0.325\textwidth]{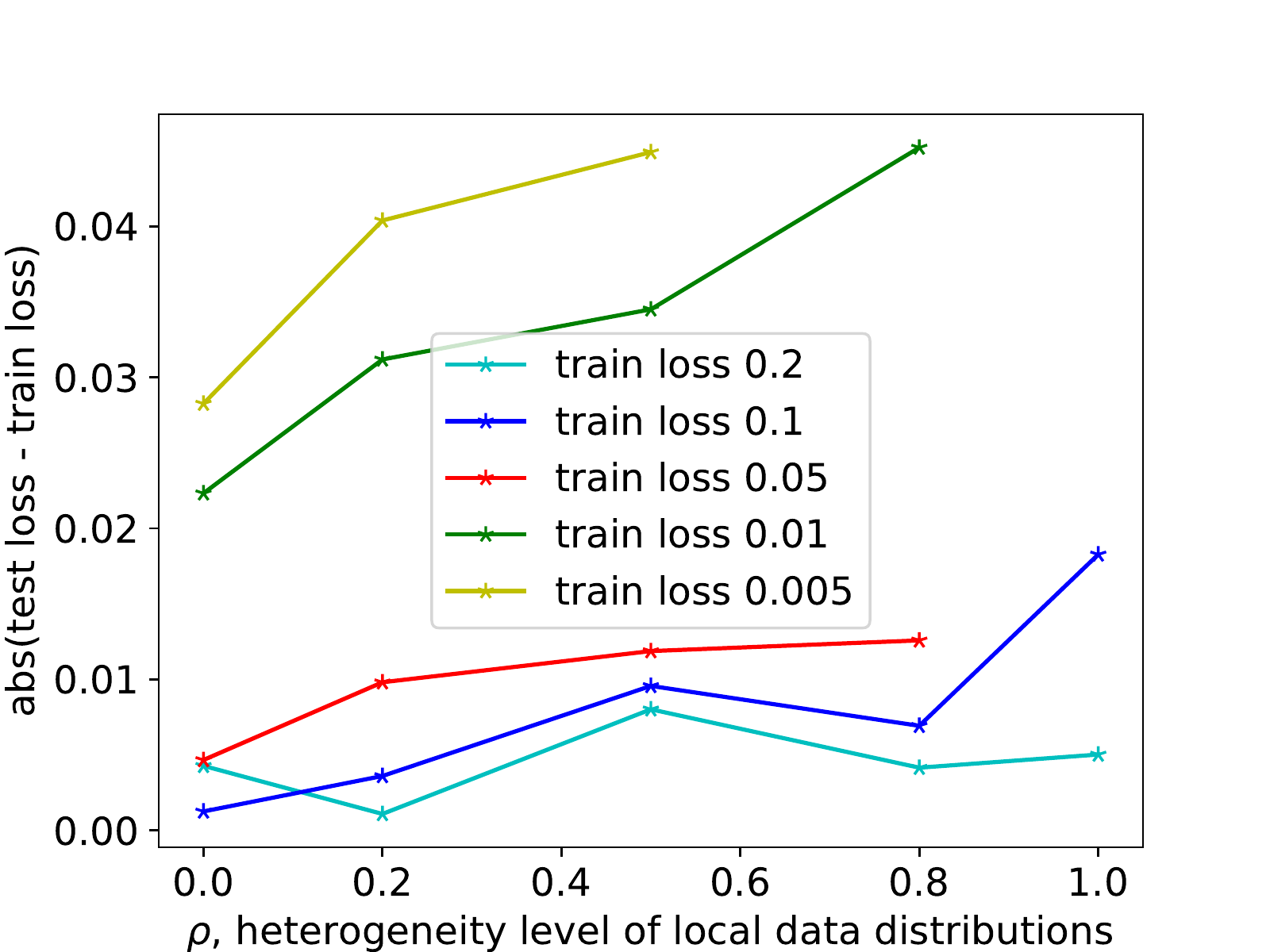}}
    \subfigure[SCAFFOLD]{\includegraphics[width=0.325\textwidth]{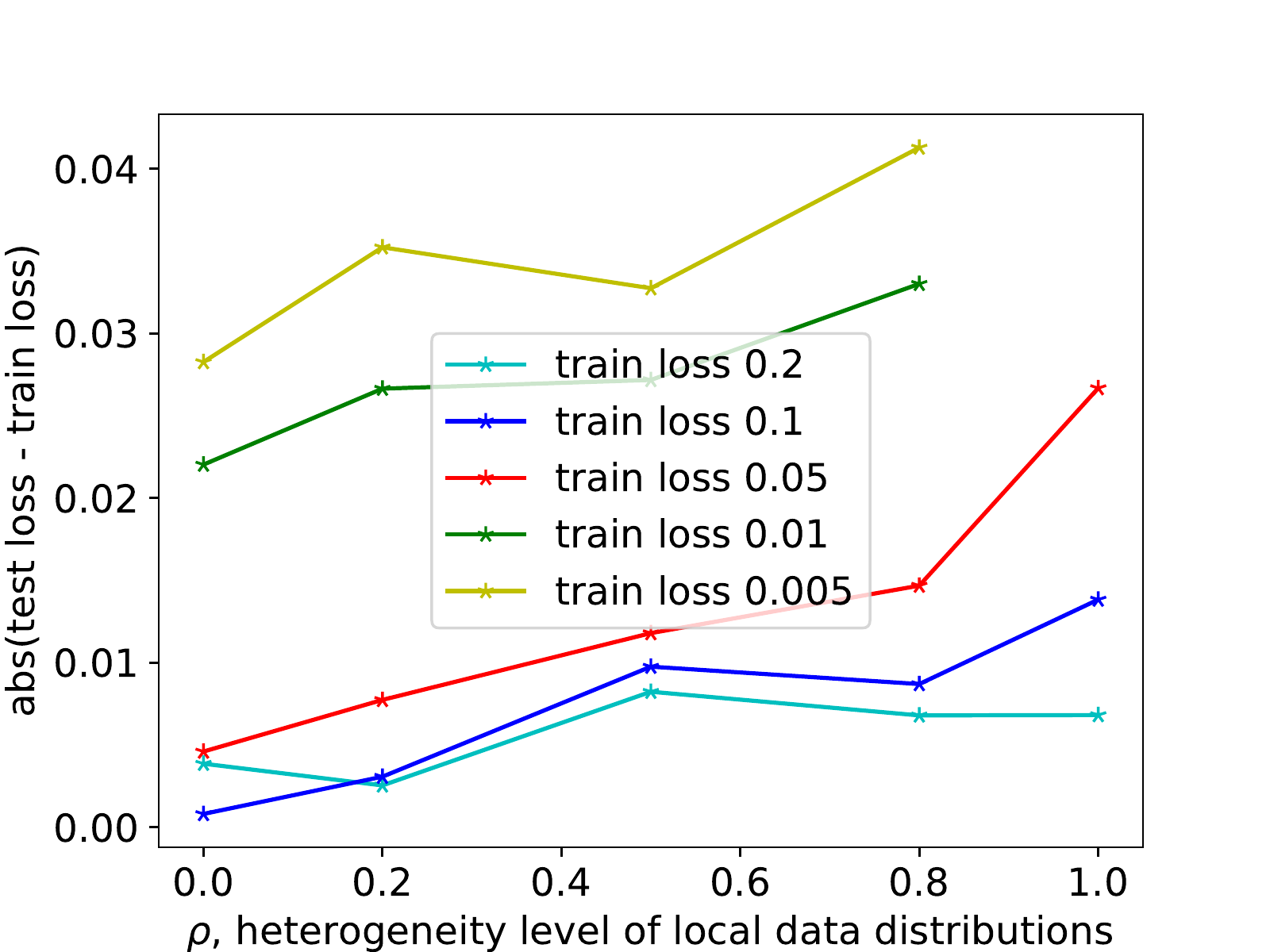}}
    \subfigure[FedProx]{\includegraphics[width=0.325\textwidth]{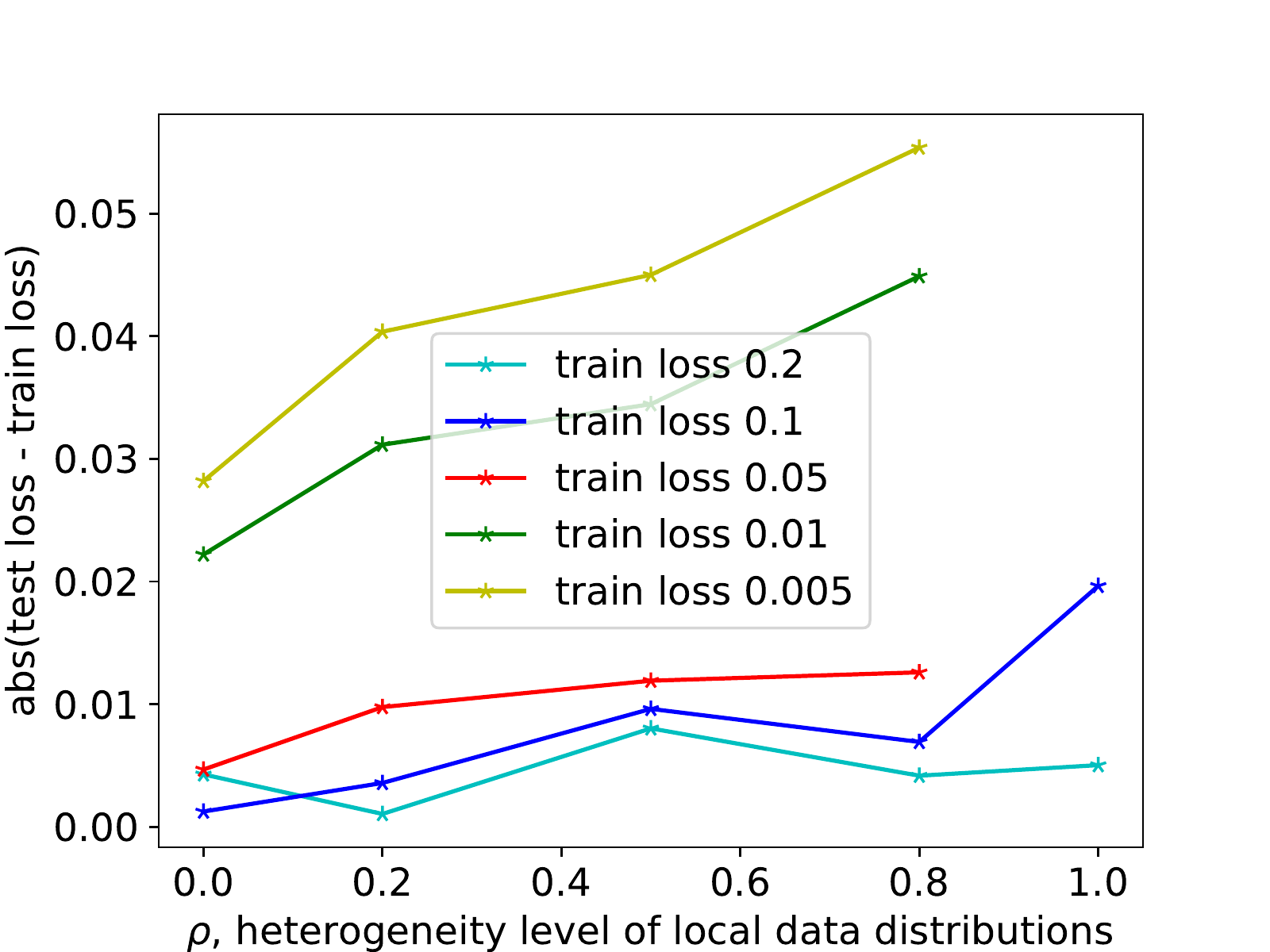}}
    \caption{Generalization errors of FedAvg,  SCAFFOLD, and FedProx}
    \label{fig:gen}
\end{figure}

{\bf Numerical Results.} The generalization errors of FedAvg, SCAFFOLD, and FedProx are shown in Fig. \ref{fig:gen}. The x-axis shows the heterogeneity level of local data distributions ($\rho$) and the y-axis shows the generalization errors of the algorithms. We note that the algorithms in some heterogeneous cases ($\rho=0.8$ or $\rho=1$) did not achieve some levels of training losses (e.g. $0.005$) before they terminated. So there are less than $5$ points in the corresponding training loss curves. The figure shows that the generalization error increases as data heterogeneity increases, which is aligned with our theoretical results.  Moreover, vertically, the generalization error also increases as the training loss level decreases. Noting that a smaller training loss generally needs more iterations in the training process. Hence these numerical results are also consistent with our bounds, which implies the generalization errors increase as $T$ gets bigger. 

\section{Conclusion}    \label{sec_conclusion}

In this paper, we provide generalization upper bounds for FedAvg, SCAFFOLD and FedProx by means of on-average stability under both convex and non-convex loss functions. Our bounds explicitly capture the effect of data heterogeneity and algorithm convergence properties on generalization performances of different algorithms. In particular, under the i.i.d. case, FedAvg reduces to the SGD method and our results are shown to be consistent to those of SGD methods.

\clearpage

\appendix

\section{Federated Learning Algorithms} \label{apx_algorithms}
In this section, we summarize FedAvg, SCAFFOLD and FedProx in detail in Algorithms \ref{alg_FedAvg},\ref{alg_SCAFFOLD},\ref{alg_FedProx}, respectively.

\begin{algorithm}
    \caption{FedAvg} \label{alg_FedAvg}
    \begin{algorithmic}[1]
        \Require $\theta^0$ as initialization of the server
        \For {$t=0,1,\dots,T-1$}      
            \State $\theta^{t+1}_{i,0} = \theta^{t},$~$\forall i=1,\dots,m$\qquad
            \For {$k=0,1,\dots,K_i-1$} (in parallel for all agents)
                \State $\theta_{i, k+1}^{t+1} = \theta_{i, k}^{t+1} - \alpha_{i,k} \nabla g_i(\theta_{i,k}^{t+1}) $
            \EndFor
            \State $\theta^{t+1} =  \sum_{i=1}^m p_i\theta^{t+1}_{i,K_i} $
        \EndFor
        \Ensure $\theta^T$ given by the server
    \end{algorithmic}
\end{algorithm}

\begin{algorithm}
    \caption{SCAFFOLD} \label{alg_SCAFFOLD}
    \begin{algorithmic}[1]
        \Require $\theta^0$ as initialization of the server
        \For {$t=0,1,\dots,T-1$}      
            \State Server broadcasts $\theta^t$
            \State Agents compute $\nabla g_i(\theta^t)$ and send it to the server
            \State Server computes $g(\theta^t)=\frac{1}{m}\sum_{i=1}^m g_i(\theta^t)$ and broadcasts it
            \State Each agent $i$ for $i=1,\dots,m$ sets
 $\theta^{t+1}_{i,0} = \theta^{t}$
            \For {$k=0,1,\dots,K_i-1$} (in parallel for all agents)
                \State $\theta_{i, k+1}^{t+1} = \theta_{i, k}^{t+1} - \alpha_{i,k} \big(\nabla g_i(\theta_{i,k}^{t+1}) - g_i(\theta^t) + g(\theta^t) \big) $
            \EndFor
            \State $\theta^{t+1} =  \sum_{i=1}^m p_i\theta^{t+1}_{i,K_i} $
        \EndFor
        \Ensure $\theta^T$ given by the server
    \end{algorithmic}
\end{algorithm}

\begin{algorithm}
    \caption{FedProx} \label{alg_FedProx}
    \begin{algorithmic}[1]
        \Require $\theta^0$ as initialization of the server
        \For {$t=0,1,\dots,T-1$}      
                \State $\theta_i^{t+1} = \arg \min_{\theta} \hat{R}_{\mathcal{S}_i}(\theta) + \frac{1}{2\eta_i} \Vert \theta - \theta^t \Vert^2 $ (in parallel for all agents)
            \State $\theta^{t+1} = \sum_{i=1}^m p_i\theta^{t+1}_i $
        \EndFor
        \Ensure $\theta^T$ given by the server
    \end{algorithmic}
\end{algorithm}

\section{Proof of Theorem \ref{thm_stability-gen}}  \label{apx_proof-gen-stability}
In this section, we provide the proof of Theorem \ref{thm_stability-gen}.

Given $\mathcal{S}$ and $\mathcal{S}^{(i)}$ which are neighboring datasets defined in Definition \ref{def_neighboring-data},
\begin{eqnarray}
        \mathbb{E}_{\mathcal{S}} \left[ \hat{R}_{\mathcal{S}_i}(\mathcal{A(S)}) \right] &=& \mathbb{E}_{\mathcal{S}} \left[\frac{1}{n_i} \sum_{j=1}^{n_i} l(\mathcal{A(S)}; z_{i,j}) \right]    \nonumber   \\
        &=& \frac{1}{n_i} \sum_{j=1}^{n_i} \mathbb{E}_{\mathcal{S}} \left[ l(\mathcal{A(S)}; z_{i,j}) \right]   \nonumber   \\
        &=& \frac{1}{n_i} \sum_{j=1}^{n_i} \mathbb{E}_{\mathcal{S}, z'_{i,j}} \left[ l(\mathcal{A}(\mathcal{S}^{(i)}); z'_{i,j}) \right] .    \nonumber
    \end{eqnarray}
    Moreover, we have
    \begin{equation}
        \mathbb{E}_{\mathcal{S}} \left[ R_i(\mathcal{A(S)}) \right] = \frac{1}{n_i} \sum_{j=1}^{n_i} \mathbb{E}_{\mathcal{S}, z'_{i,j}} \left[ l(\mathcal{A(S)}; z'_{i,j}) \right] ,    \nonumber 
    \end{equation}
    since $z'_{i,j}$ and $\mathcal{S}$ are independent for any $j$.
    Thus,
    \begin{eqnarray}
        \mathbb{E}_{\mathcal{A,S}} \left[  R(\mathcal{A(S)}) - \hat{R}(\mathcal{A(S)})  \right] &\le& \mathbb{E}_{\mathcal{A,S}} \left[ \sum_{i=1}^m \frac{n_i}{n} \left( R_i(\mathcal{A(S)}) - \hat{R}_{\mathcal{S}_i}(\mathcal{A(S)}) \right) \right]    \nonumber   \\
        &=& \sum_{i=1}^m \frac{n_i}{n} \mathbb{E}_{\mathcal{A}} \left[ \frac{1}{n_i} \sum_{j=1}^{n_i} \mathbb{E}_{\mathcal{S}, z'_{i,j}} \left( l(\mathcal{A(S)}; z'_{i,j}) - l(\mathcal{A}(\mathcal{S}^{(i)}); z'_{i,j}) \right) \right]  \nonumber   \\
        &\le& \epsilon ,   \nonumber
    \end{eqnarray}
    where the last inequality follows Definition \ref{def_stability}.
    This completes the proof.

\section{Generalization Bounds for Convex Losses} 
 \label{apx_proof-convex}

In this section, we drop index $t$ when context is clear for simplicity. We first provide the bound involving data heterogeneity by means of total variation between local distribution and global one.
\begin{lemma}   \label{lmm_TV-bound}
    Under Assumption \ref{assump_Lip-continuous} and given $i \in [m]$, for any $\theta$ we have
    $$
        \Vert \nabla R_i(\theta) - \nabla R(\theta) \Vert \le 2L D_i,
    $$
    where $D_i = d_{TV}(P_i, P)$ with $P = \sum_{i=1}^m p_i P_i$.
\end{lemma}
\begin{proof}
    Let $\mathcal{Z}_i$ and $\mathcal{Z}$ be the supports of $P_i$ and $P$, respectively.
    \begin{eqnarray}
        \Vert \nabla R_i(\theta) - \nabla R(\theta) \Vert &=& \Vert \nabla_{\theta} \int_{\mathcal{Z}_i} l(\theta;z) d P_i(z) - \nabla_{\theta} \int_{\mathcal{Z}} l(\theta;z) d P(z)  \Vert  \nonumber   \\
        &=& \Vert \int_{\mathcal{Z}_i \cup \mathcal{Z}} \big(\nabla_{\theta}l(\theta;z) d P_i(z) - \nabla_{\theta}l(\theta;z) d P(z) \big) \Vert \nonumber   \\
        &\le& \int_{\mathcal{Z}_i \cup \mathcal{Z}} \Vert \nabla_{\theta}l(\theta;z) d P_i(z) - \nabla_{\theta}l(\theta;z) d P(z) \Vert \nonumber   \\
        &=& \int_{\mathcal{Z}_i \cup \mathcal{Z}} \Vert \nabla l(\theta;z) \Vert \Vert dP_i(z) - dP(z) \Vert    \nonumber   \\
        &\le& \int_{\mathcal{Z}_i \cup \mathcal{Z}} L | dP_i(z) - dP(z) |   \nonumber   \\
        &=& 2L d_{TV} (P_i, P)  \nonumber
    \end{eqnarray}
    by noting the definition of total variation of two distributions $P$ and $Q$ is 
    $$
        d_{TV}(P,Q) = \frac{1}{2}\int |dP - dQ|.
    $$
\end{proof}

When the loss function is convex, the gradient descent operator has the non-expansiveness property stated by the following lemma.
\begin{lemma}   \label{lmm_nonexp}
    Suppose $f(x)$ is a $\beta$-Lipschitz smooth, convex function with respect to $x$. Consider gradient descent operator $G_{\alpha}(x) := x - \alpha \nabla f(x)$. Then, for $\alpha \le 1/\beta$,
    $$
        \Vert G_{\alpha}(x) - G_{\alpha}(y) \Vert \le \Vert x - y \Vert.
    $$
\end{lemma}
\begin{proof}
    Since $f$ is $\beta$-smooth and convex, we know that
    $$
        \langle \nabla f(x) - \nabla f(y), x - y \rangle \ge \frac{1}{\beta} \Vert \nabla f(x) - \nabla f(y) \Vert^2 .
    $$
    Using this fact,
    \begin{eqnarray}
        \Vert G_{\alpha}(x) - G_{\alpha}(y) \Vert^2 &=& \Vert x - y -\alpha(\nabla f(x) - \nabla f(y)) \Vert^2  \nonumber   \\
        &=& \Vert x - y \Vert^2 + \alpha^2 \Vert \nabla f(x) - \nabla f(y) \Vert^2 - \alpha \langle \nabla f(x) - \nabla f(y), x - y \rangle  \nonumber   \\
        &\le& \Vert x-y \Vert^2 + \alpha (\alpha - \beta^{-1}) \Vert \nabla f(x) - \nabla f(y) \Vert^2  \nonumber   \\
        &\le& \Vert x-y \Vert^2 \nonumber
    \end{eqnarray}
    when $\alpha \le 1/\beta$.
\end{proof}

The proximal operator is also non-expansive, which is shown by the following lemma.
\begin{lemma}   \label{lmm_prox-convex}
    Suppose $f$ is convex. Define the proximal operator by 
    $$
        \mathrm{prox}_{f}(x) := \arg\min_{y} f(y) + \frac{1}{2}\Vert y - x \Vert^2.
    $$
    Then, for any $x_1$, $x_2$, we have
    $$
        \Vert \mathrm{prox}_{f}(x_1) - \mathrm{prox}_{f}(x_2) \Vert \le \Vert x_1 - x_2 \Vert.
    $$
\end{lemma}
\begin{proof}
    Let $u_1 = \mathrm{prox}_f(x_1)$ and $u_2 = \mathrm{prox}_f(x_2)$. According to the first-order optimality condition, we have
    \begin{eqnarray}
        \nabla f(u_1) +  u_1 - x_1  &=& 0 \nonumber   \\
        \nabla f(u_2) +  u_2 - x_2  &=& 0   \nonumber
    \end{eqnarray}
    Since $f$ is convex, we further have
    \begin{eqnarray*}
        0 &\le& \langle \nabla f(u_1) - \nabla f(u_2), u_1 - u_2 \rangle    \\
        &=& \langle x_1 - u_2 - (x_2 - u_2), u_1 - u_2 \rangle  \\
        &=& \langle x_1 - x_2, u_1 - u_2 \rangle - \Vert u_1 - u_2 \Vert^2
    \end{eqnarray*}
    and hence
    $$
        \Vert u_1 - u_2 \Vert^2 \le \langle x_1 - x_2, u_1 - u_2 \rangle \le \Vert x_1 - x_2 \Vert \Vert u_1 - u_2 \Vert
    $$
    which completes the proof.
\end{proof}

\subsection{Analysis for FedAvg under convex losses}

\begin{lemma}   \label{lmm_drift-FedAvg}
    Suppose Assumptions \ref{assump_Lip-continuous}-\ref{assump_convexity} hold. Then for FedAvg with $\alpha_{i,k} \le 1/\beta$,
    $$
        \mathbb{E}\Vert \theta_{i,k} - \theta_t \Vert \le \tilde{\alpha}_{i,t} \big( \mathbb{E}\Vert \nabla R(\theta_t) \Vert + 2LD_i + \sigma \big), ~\forall k=1,\dots,K_i ,
    $$
    where $\tilde{\alpha}_{i,t} = \sum_{k=0}^{K_i-1} \alpha_{i,k}$.
\end{lemma}
\begin{proof}
    Considering local update \eqref{eq_FedAvg-update} of FedAvg
    \begin{eqnarray}
        \mathbb{E}\Vert \theta_{i,k+1} - \theta_t \Vert &=& \mathbb{E} \Vert \theta_{i,k} - \alpha_{i,k} g_i(\theta_{i,k}) - \theta_t \Vert   \nonumber   \\
        &\le& \mathbb{E}\Vert \theta_{i,k} - \theta_t - \alpha_{i,k} (g_i(\theta_{i,k}) - g_i(\theta_t))  \Vert + \alpha_{i,k}\mathbb{E}\Vert g_i(\theta_t) \Vert   \nonumber   \\
        &\overset{(a)}{\le}& \mathbb{E}\Vert \theta_{i,k} - \theta_t \Vert + \alpha_{i,k}\mathbb{E}\Vert g_i(\theta_t) \Vert    \nonumber   \\
        &\le& \mathbb{E}\Vert \theta_{i,k} - \theta_t \Vert + \alpha_{i,k}(\mathbb{E}\Vert g_i(\theta_t) - \nabla R_i(\theta_t) \Vert + \mathbb{E}\Vert \nabla R_i(\theta_t) \Vert)   \nonumber   \\
        &\overset{(b)}{\le}& \mathbb{E}\Vert \theta_{i,k} - \theta_t \Vert + \alpha_{i,k}(\mathbb{E}\Vert \nabla R_i(\theta_t) \Vert + \sigma),    \nonumber
    \end{eqnarray}
    where $(a)$ follows Lemma \ref{lmm_nonexp}; $(b)$ follows Assumption \ref{assump_bounded-grad-var}. Unrolling the above and noting $\theta_{i,0} = \theta_t$ yields
    \begin{eqnarray}
        \mathbb{E}\Vert \theta_{i,k} - \theta_t \Vert &\le& \mathbb{E}\Vert \theta_{i,0} - \theta_t \Vert + \sum_{l=0}^{k-1}\alpha_{i,l} \big( \mathbb{E}\Vert \nabla R_i(\theta_t) \Vert + \sigma \big)  \nonumber   \\
        &\le& \sum_{l=0}^{K_i-1} \alpha_{i,l} \big( \mathbb{E}\Vert \nabla R_i(\theta_t) \Vert + \sigma \big)   \nonumber   \\
        &=& \tilde{\alpha}_i \big( \mathbb{E}\Vert \nabla R_i(\theta_t) \Vert + \sigma \big)  \nonumber   \\
        &\le& \tilde{\alpha}_i \big( \mathbb{E}\Vert \nabla R(\theta_t) \Vert + 2LD_i + \sigma \big),   \nonumber
    \end{eqnarray}
    where the last inequality follows Lemma \ref{lmm_TV-bound}.
\end{proof}

\begin{lemma}   \label{lmm_grad-bnd-FedAvg}
    Given Assumptions \ref{assump_Lip-continuous}-\ref{assump_convexity} and considering \eqref{eq_FedAvg-update} of FedAvg, for $\alpha_{i,k} \le 1/\beta$ we have
    $$
        \mathbb{E}\Vert g_i(\theta_{i,k}) \Vert \le (1 + \beta \tilde{\alpha}_{i,t}) \big( \mathbb{E}\Vert \nabla R(\theta_t) \Vert + 2LD_i + \sigma \big),
    $$
    where $g_i(\cdot)$ is the sampled gradient of client $i$, $\tilde{\alpha}_{i,t} = \sum_{k=0}^{K_i-1} \alpha_{i,k}$.
\end{lemma}
\begin{proof}
    Using Lemmas \ref{lmm_TV-bound} and \ref{lmm_drift-FedAvg}, we obtain
    \begin{eqnarray}
        \mathbb{E} \Vert g_i(\theta_{i,k}) \Vert &\le& \mathbb{E} \Vert g_i(\theta_{i,k}) - \nabla R_i(\theta_{i,k}) \Vert + \mathbb{E} \Vert \nabla R_i(\theta_{i,k}) \Vert  \nonumber   \\
        &\le& \mathbb{E}\Vert \nabla R_i(\theta_{i,k}) \Vert + \sigma   \nonumber   \\
        &\le& \mathbb{E}\Vert \nabla R_i(\theta_t) \Vert + \mathbb{E}\Vert \nabla R_i(\theta_{i,k}) - \nabla R_i(\theta_t) \Vert + \sigma   \nonumber   \\
        &\le& \mathbb{E}\Vert \nabla R(\theta_t) \Vert + \mathbb{E}\Vert \nabla R_i(\theta_t) - \nabla R(\theta_t) \Vert + \beta \mathbb{E}\Vert \theta_{i,k} - \theta_t \Vert + \sigma \nonumber   \\
        &\le& (1 + \beta \tilde{\alpha}_i) \big( \mathbb{E}\Vert \nabla R(\theta_t) \Vert + 2LD_i + \sigma \big) . \nonumber
    \end{eqnarray}
\end{proof}

\begin{theorem}[FedAvg part of Theorem \ref{thm_gen-convex}]    \label{thm_gen-FedAvg-convex}
    Suppose Assumptions \ref{assump_Lip-continuous}-\ref{assump_convexity} hold and consider FedAvg (Algorithm \ref{alg_FedAvg}). Let $\{ \theta_t \}_{t=0}^{T}$ and $\{ \theta'_t \}_{t=0}^T$ be two trajectories of the server induced by neighboring datasets $\mathcal{S}$ and $\mathcal{S}^{(i)}$, respectively. Suppose $\theta_0 = \theta'_0$. Then,
    $$
        \mathbb{E}\Vert \theta_T - \theta'_T \Vert \le \frac{2}{n}\sum_{t=0}^{T-1} \tilde{\alpha}_{i,t}(1+\beta \tilde{\alpha}_{i,t}) \big( 2LD_i + \mathbb{E}\Vert \nabla R(\theta_t) \Vert + \sigma \big),
    $$
    where $\tilde{\alpha}_{i,t}=\sum_{k=0}^{K_i-1}\alpha_{i,k}$ and $D_i=d_{TV}(P_i, P)$.
\end{theorem}
\begin{proof}
    Note that in the phase of local update, each client runs stochastic gradient descent (SGD) using its own local gradient $g_i(\cdot)$ sampled uniformly from its dataset. Given time index $t$, for client $j$ with $j \ne i$, the local datasets are identical since the perturbed data point only occurs at client $i$. Thus, when $j\ne i$, we have for any $k=0,\dots,K_j-1$,
    \begin{eqnarray}
        \mathbb{E}\Vert \theta_{j,k+1} - \theta'_{j, k+1} \Vert &=& \mathbb{E}\Vert \theta_{j,k} - \theta'_{j,k} - \alpha_{j,k}(g_j(\theta_{j,k}) - g_j(\theta'_{j,k})) \Vert   \nonumber   \\
        &\le& \mathbb{E}\Vert \theta_{j,k} - \theta'_{j,k} \Vert    \nonumber
    \end{eqnarray}
    where we use Lemma \ref{lmm_nonexp} in the last inequality. Here we drop $t$ for simplicity. Unrolling it gives
    \begin{eqnarray}    \label{eq_iter-j-FedAvg-convex}
        \mathbb{E}\Vert \theta_{j, K_j} - \theta'_{j, K_j} \Vert \le \mathbb{E}\Vert \theta_t - \theta'_t \Vert, ~~ \forall j\ne i.
    \end{eqnarray}

    For client $i$, there are two cases to consider. In the first case, SGD selects the index of an sample at local step $k$ on which is identical in $\mathcal{S}$ and $\mathcal{S}^{(i)}$. In this sense, we have 
    $$
        \Vert \theta_{i, k+1} - \theta'_{i, k+1} \Vert \le \Vert \theta_{i,k} - \theta'_{i,k} \Vert
    $$
    due to the non-expansiveness of gradient descent operator by Lemma \ref{lmm_nonexp}. And this case happens with probability $1 - 1/n_i$ (since only one sample is perturbed for client $i$).

    In the second case, SGD encounters the perturbed sample at local time step $k$, which happens with probability $1/n_i$. We denote the gradient of this perturbed sample as $g'_i(\cdot)$. Then,
    \begin{eqnarray}
        \Vert \theta_{i,k+1} - \theta'_{i,k+1} \Vert &=& \Vert \theta_{i,k} - \theta'_{i,k} - \alpha_{i,k}(g_i(\theta_{i,k}) - g'_i(\theta'_{i,k})) \Vert   \nonumber   \\
        &\le& \Vert \theta_{i,k} - \theta'_{i,k} - \alpha_{i,k}(g_i(\theta_{i,k}) - g_i(\theta'_{i,k})) \Vert + \alpha_{i,k}\Vert g_i(\theta'_{i,k}) - g'_i(\theta'_{i,k}) \Vert    \nonumber   \\
        &\le& \Vert \theta_{i,k} - \theta'_{i,k} \Vert + \alpha_{i,k}\Vert g_i(\theta'_{i,k}) - g'_i(\theta'_{i,k}) \Vert .  \nonumber
    \end{eqnarray}
    Combining these two cases we have for client $i$
    \begin{eqnarray}
        \mathbb{E}\Vert \theta_{i,k+1} - \theta'_{i,k+1} \Vert &\le& \mathbb{E}\Vert \theta_{i,k} - \theta'_{i,k} \Vert + \frac{\alpha_{i,k}}{n_i}\mathbb{E}\Vert g_i(\theta'_{i,k}) - g'_i(\theta'_{i,k}) \Vert  \nonumber   \\
        &\le& \mathbb{E}\Vert \theta_{i,k} - \theta'_{i,k} \Vert + \frac{2\alpha_{i,k}}{n_i} \mathbb{E}\Vert g_i(\theta_{i,k}) \Vert ,   \nonumber  
    \end{eqnarray}
    where the last inequality follows that $g_i(\cdot)$ and $g'_i(\cdot)$ are sampled from the same distribution. Then unrolling it we have
    \begin{eqnarray}    \label{eq_iter-i-FedAvg-convex}
        \mathbb{E}\Vert \theta_{i,K_i} - \theta'_{i, K_i} \Vert \le \mathbb{E}\Vert \theta_t - \theta'_t \Vert + \frac{2}{n_i}\sum_{k=0}^{K_i-1}\alpha_{i,k}\mathbb{E}\Vert g_i(\theta_{i,k}) \Vert .
    \end{eqnarray}
    Combining \eqref{eq_iter-j-FedAvg-convex} and \eqref{eq_iter-i-FedAvg-convex} gives
    \begin{eqnarray}
        \mathbb{E}\Vert \theta_{t+1} - \theta'_{t+1} \Vert &=& \mathbb{E} \Vert \sum_{j=1}^m p_j (\theta_{j,K_j} - \theta'_{j,K_j}) \Vert   \nonumber   \\
        &\le& \sum_{j=1}^m p_j \mathbb{E}\Vert \theta_{j,K_j} - \theta'_{j,K_j} \Vert   \nonumber   \\
        &\le& \mathbb{E}\Vert \theta_t - \theta'_t \Vert + \frac{2p_i}{n_i}\sum_{k=0}^{K_i-1}\alpha_{i,k}\mathbb{E}\Vert g_i(\theta_{i,k}) \Vert   \nonumber   \\
        &\le& \mathbb{E}\Vert \theta_t - \theta'_t \Vert + \frac{2}{n} \tilde{\alpha}_{i,t}(1 + \beta \tilde{\alpha}_{i,t}) \big( \mathbb{E}\Vert \nabla R(\theta_t) \Vert + 2LD_i + \sigma \big)   \nonumber ,
    \end{eqnarray}
    where we use Lemma \ref{lmm_grad-bnd-FedAvg} in the last step. Iterating the above over $t$ and noting $\theta_0 = \theta'_0$, we conclude the proof.
    
\end{proof}

\subsection{Analysis for SCAFFOLD under convex losses}

\begin{lemma}   \label{lmm_drift-SCAFFOD-convex}
    Suppose Assumptions \ref{assump_Lip-continuous}-\ref{assump_convexity} hold. Running SCAFFOLD with $\alpha_{i,k} \le 1/\beta$, then for any $i \in [m]$
    $$
        \mathbb{E}\Vert \theta_{i,k} - \theta_t \Vert \le \tilde{\alpha}_{i,t} (\mathbb{E}\Vert R(\theta_t)\Vert + \sigma), ~~ \forall k=1,\dots,K_i 
    $$
    where $\tilde{\alpha}_{i,t}=\sum_{k=0}^{K_i-1}\alpha_{i,k}$.
\end{lemma}
\begin{proof}
    Considering local update \eqref{eq_SCAFFOLD-update} of SCAFFOLD
    \begin{eqnarray}
        \mathbb{E}\Vert \theta_{i,k+1} - \theta_t \Vert &=& \mathbb{E}\Vert \theta_{i,k} - \alpha_{i,k}(g_i(\theta_{i,k}) - g_i(\theta_t) + g(\theta_t)) - \theta_t \Vert \nonumber   \\
        &\le& \mathbb{E}\Vert \theta_{i,k} - \theta_t - \alpha_{i,k}(g_i(\theta_{i,k}) - g_i(\theta_t)) \Vert + \alpha_{i,k}\mathbb{E}\Vert g(\theta_t) \Vert   \nonumber   \\
        &\le& \mathbb{E}\Vert \theta_{i,k} - \theta_t \Vert + \alpha_{i,k}\mathbb{E}\Vert g(\theta_t) \Vert \nonumber   \\
        &\le& \mathbb{E}\Vert \theta_{i,k} - \theta_t \Vert + \alpha_{i,k}(\mathbb{E}\Vert R(\theta_t) \Vert + \sigma ) \nonumber
    \end{eqnarray}
    where we use the non-expansiveness property of gradient descent operator and Assumption \ref{assump_bounded-grad-var}. Therefore, for any $k=1,\dots,K_i-1$,
    \begin{eqnarray}
        \mathbb{E}\Vert \theta_{i,k} - \theta_t \Vert &\le& \sum_{l=0}^{k-1} \alpha_{i,k}(\mathbb{E}\Vert R(\theta_t) \Vert + \sigma )  \nonumber   \\
        &\le& \tilde{\alpha}_{i,t}(\mathbb{E}\Vert R(\theta_t) \Vert + \sigma ),    \nonumber
    \end{eqnarray}
    which completes the proof.
\end{proof}

\begin{lemma}   \label{lmm_grad-bnd-SCAFFOLD-convex}
    Given Assumptions \ref{assump_Lip-continuous}-\ref{assump_convexity} and considering SCAFFOLD (Algorithm \ref{alg_SCAFFOLD}), with $\alpha_{i,k}\le 1/\beta$ we have the following inequalities
    \begin{eqnarray}
        \mathbb{E}\Vert g_i(\theta_{i,k}) \Vert &\le& (1+\beta \tilde{\alpha}_{i,t})(\mathbb{E}\Vert \nabla R(\theta_t) \Vert + \sigma) + 2L D_i,   \nonumber  \\
        \mathbb{E}\Vert g_i(\theta_t) \Vert &\le& 2LD_i + \mathbb{E}\Vert \nabla R(\theta_t) \Vert + \sigma   \nonumber
    \end{eqnarray}
    for any $i \in [m]$, $k=0,\dots,K_i-1$ and $t=0,1,\dots$.
\end{lemma}
\begin{proof}
    Note that based on Assumption \ref{assump_bounded-grad-var},
    \begin{eqnarray}
        \mathbb{E}\Vert g_i(\theta_{i,k}) \Vert &\le& \mathbb{E}\Vert \nabla R_i(\theta_{i,k}) \Vert + \sigma   \nonumber   \\
        &\le& \mathbb{E}\Vert \nabla R_i(\theta_{i,k}) - \nabla R_i(\theta_t) \Vert + \mathbb{E}\Vert \nabla R_i(\theta_t) \Vert + \sigma   \nonumber   \\
        &\le& \beta \mathbb{E}\Vert \theta_{i,k} - \theta_t \Vert + \mathbb{E}\Vert \nabla R_i(\theta_t) - \nabla R(\theta_t) \Vert + \mathbb{E}\Vert \nabla R(\theta_t) \Vert + \sigma \nonumber   \\
        &\le& (1+\beta \tilde{\alpha}_{i,t})(\mathbb{E}\Vert \nabla R(\theta_t) \Vert + \sigma) + 2L D_i,   \nonumber
    \end{eqnarray}
    where we use Lemmas \ref{lmm_TV-bound} and \ref{lmm_drift-SCAFFOD-convex}.

    Similarly, using same techniques we have
    \begin{eqnarray}
        \mathbb{E}\Vert g_i(\theta_t) \Vert &\le& \mathbb{E}\Vert \nabla R_i(\theta_t) \Vert + \sigma   \nonumber   \\
        &\le& \mathbb{E}\Vert \nabla R_i(\theta_t) - \nabla R(\theta_t) \Vert + \mathbb{E}\Vert \nabla R(\theta_t) \Vert + \sigma   \nonumber   \\
        &\le& 2LD_i + \mathbb{E}\Vert \nabla R(\theta_t) \Vert + \sigma .  \nonumber
    \end{eqnarray}
\end{proof}

\begin{theorem}[SCAFFOLD part of Theorem \ref{thm_gen-convex}]  \label{thm_gen-SCAFFOLD-convex}
    Suppose Assumptions \ref{assump_Lip-continuous}-\ref{assump_convexity} hold and consider SCAFFOLD (Algorithm \ref{alg_SCAFFOLD}). Let $\{ \theta_t \}_{t=0}^{T}$ and $\{ \theta'_t \}_{t=0}^T$ be two trajectories of the server induced by neighboring datasets $\mathcal{S}$ and $\mathcal{S}^{(i)}$, respectively. Suppose $\theta_0 = \theta'_0$. Then
    \begin{align*}
        \mathbb{E}\Vert \theta_T - \theta'_T \Vert \le \frac{2}{n} \sum_{t=0}^{T-1} \exp{\Big(2\beta \sum_{l=t+1}^{T-1} \hat{\alpha}_l \Big)} \Big( 2L D_i \gamma^1_{t} + \gamma^2_{t} \mathbb{E}\Vert \nabla R(\theta_{t}) \Vert  + \sigma \gamma^2_{t} \Big)
    \end{align*}
    where 
        $$\gamma^1_t := 2 \tilde{\alpha}_{i,t} + \hat{\alpha}_t , ~~ \gamma^2_t := \gamma^1_t + \beta \tilde{\alpha}^2_{i,t}$$
    with $\tilde{\alpha}_{i,t} := \sum_{k=0}^{K_i - 1} \alpha_{i,k}$,  $\hat{\alpha}_t := \sum_{j=1}^m p_j \tilde{\alpha}_{j,t}$, and $\sum_{l=T}^{T-1} \hat{\alpha}_l = 0, \forall \hat{\alpha}_l$.
\end{theorem}
\begin{proof}
    Similar to the idea used in the proof of Theorem \ref{thm_gen-FedAvg-convex}, given time index $t$ and client $j$ with $j\ne i$, note that the local gradients $g_j(\cdot)$ are identical for client $j$ in the sense that local datasets for client $j$ are the same. However, since SCAFFOLD uses the global sampled gradient $g(\cdot)$ during the local update, it is still possible to encounter the perturbed sample. Thus, for $j \ne i$, we distinguish two cases. In the first case, SCAFFOLD does not sample the perturbed gradient of client $i$, i.e., $g(\cdot) = g'(\cdot)$ at local step $k$. Then, with probability equal to $1 - 1/n_i$
    \begin{eqnarray}
        \Vert \theta_{j,k+1} - \theta'_{j,k+1} \Vert &\le& \Vert \theta_{j,k} - \theta'_{j,k} - \alpha_{j,k}(g_j(\theta_{j,k}) - g_j(\theta'_{j,k}))\Vert + \alpha_{j,k}\Vert g_j(\theta_t) - g_j(\theta'_t) \Vert  \nonumber   \\
        && + \alpha_{j,k}\Vert g(\theta_t) - g(\theta'_t)\Vert  \nonumber   \\
        &\le& \Vert \theta_{j,k} - \theta'_{j,k} \Vert + 2\alpha_{j,k}\beta \Vert \theta_t - \theta'_t \Vert    \nonumber   
    \end{eqnarray}
    where the second inequality follows Lemma \ref{lmm_nonexp} and Assumption \ref{assump_Lip-smooth}.

    In the second case, the perturbed data point of client $i$ is sampled to calculate the global gradient $g'(\cdot)$, meaning $g(\cdot) - g'(\cdot) = p_i(g_i(\cdot) - g'_i(\cdot))$, where we denote the gradient evaluated at the perturbed sample as $g'_i(\cdot)$. This happens with probability $1/n_i$ and hence we have
    \begin{eqnarray}    
        \Vert \theta_{j,k+1} - \theta'_{j,k+1} \Vert &\le& \Vert \theta_{j,k} - \theta'_{j,k} - \alpha_{j,k}(g_j(\theta_{j,k}) - g_j(\theta'_{j,k}))\Vert + \alpha_{j,k}\Vert g_j(\theta_t) - g_j(\theta'_t) \Vert  \nonumber   \\
        && + \alpha_{j,k}\Vert g(\theta_t) - g'(\theta'_t)\Vert  \nonumber   \\
        &\le& \Vert \theta_{j,k} - \theta'_{j,k} - \alpha_{j,k}(g_j(\theta_{j,k}) - g_j(\theta'_{j,k}))\Vert + \alpha_{j,k}\Vert g_j(\theta_t) - g_j(\theta'_t) \Vert  \nonumber   \\
        && + \alpha_{j,k}\Vert g(\theta_t) - g(\theta'_t)\Vert + \alpha_{j,k}\Vert g(\theta'_t) - g'(\theta'_t) \Vert  \nonumber   \\
        &\le& \Vert \theta_{j,k} - \theta'_{j,k} \Vert + 2\beta \alpha_{j,k}\Vert \theta_t - \theta'_t \Vert + \alpha_{j,k}p_i\Vert g_i(\theta'_t) - g'_i(\theta'_t) \Vert .   \nonumber
    \end{eqnarray}
    Combining these two cases, we conclude that for client $j$ with $j\ne i$
    \begin{eqnarray}
        \mathbb{E}\Vert \theta_{j,k+1} - \theta'_{j,k+1} \Vert &\le& \mathbb{E}\Vert \theta_{j,k} - \theta'_{j,k} \Vert + 2\beta\alpha_{j,k}\mathbb{E}\Vert \theta_t - \theta'_t \Vert + \frac{\alpha_{j,k}p_i}{n_i}\mathbb{E}\Vert g_i(\theta'_t) - g'_i(\theta'_t) \Vert  \nonumber   \\
        &\le& \mathbb{E}\Vert \theta_{j,k} - \theta'_{j,k} \Vert + 2\beta\alpha_{j,k}\mathbb{E}\Vert \theta_t - \theta'_t \Vert + \frac{2\alpha_{j,k}}{n}\mathbb{E}\Vert g_i(\theta_t) \Vert,   \nonumber   \\
        &\le& \mathbb{E}\Vert \theta_{j,k} - \theta'_{j,k} \Vert + 2\beta\alpha_{j,k}\mathbb{E}\Vert \theta_t - \theta'_t \Vert + \frac{2\alpha_{j,k}}{n}(2LD_i + \mathbb{E}\Vert \nabla R(\theta_t)\Vert + \sigma) \nonumber
    \end{eqnarray}
    where we use $p_i = n_i/n$ and $g_i, g'_i$ are drawn from the same distribution; we also use Lemma \ref{lmm_grad-bnd-SCAFFOLD-convex} in the last step. Unrolling the above over $k$ we obtain
    \begin{eqnarray}    \label{eq_iter-j-SCAFFOLD-convex}
        \mathbb{E}\Vert \theta_{j,K_j} - \theta'_{j,K_j} \Vert \le (1 + \beta \tilde{\alpha}_{j,t})\mathbb{E}\Vert \theta_t - \theta'_t \Vert + \frac{2\tilde{\alpha}_{j,t}}{n}(2LD_i + \mathbb{E}\Vert \nabla R(\theta_t)\Vert + \sigma) , ~~ \forall j\ne i .
    \end{eqnarray}

    Next, we specifically consider client $i$. Similar to the above analysis, there are two cases as well. In the first case, at local step $k$ client $i$ does not select the perturbed sample to compute the gradient. This happens with probability $1-1/n_i$. Then,
    \begin{eqnarray}
        \Vert \theta_{i,k+1} - \theta'_{i,k+1} \Vert &\le& \Vert \theta_{i,k} - \theta'_{i,k} - \alpha_{i,k}(g_i(\theta_{i,k}) - g_i(\theta'_{i,k}))\Vert + \alpha_{i,k}\Vert g_i(\theta_t) - g_i(\theta'_t) \Vert  \nonumber   \\
        && + \alpha_{i,k}\Vert g(\theta_t) - g(\theta'_t)\Vert  \nonumber   \\
        &\le& \Vert \theta_{i,k} - \theta'_{i,k} \Vert + 2\alpha_{i,k}\beta \Vert \theta_t - \theta'_t \Vert  .  \nonumber  
    \end{eqnarray}
    In the second case, the perturbed sample is selected to calculate local gradient for client $i$, which has the probability equal to $1/n_i$. Then,
    \begin{eqnarray}
        \Vert \theta_{i,k+1} - \theta'_{i,k+1} \Vert &\le& \Vert \theta_{i,k} - \theta'_{i,k} - \alpha_{i,k}(g_i(\theta_{i,k}) - g'_i(\theta'_{i,k}))\Vert + \alpha_{i,k}\Vert g_i(\theta_t) - g'_i(\theta'_t) \Vert  \nonumber   \\
        && + \alpha_{i,k}\Vert g(\theta_t) - g'(\theta'_t)\Vert  \nonumber   \\
        &\le& \Vert \theta_{i,k} - \theta'_{i,k} \Vert + \alpha_{i,k}\Vert g_i(\theta_t) - g_i(\theta'_t) \Vert + \alpha_{i,k}\Vert g(\theta_t) - g(\theta'_t)\Vert  \nonumber   \\
        && + \alpha_{i,k}\big( \Vert g_i(\theta'_{i,k}) - g'_i(\theta'_{i,k}) \Vert + (1+p_i)\Vert g_i(\theta'_t) - g'_i(\theta'_t) \Vert \big)  \nonumber \\
        &\le& \Vert \theta_{i,k} - \theta'_{i,k} \Vert + 2\beta\alpha_{i,k}\Vert \theta_t - \theta'_t \Vert + \alpha_{i,k} \Vert g_i(\theta'_{i,k}) - g'_i(\theta'_{i,k}) \Vert \nonumber   \\
        && + \alpha_{i,k}(1+p_i)\Vert g_i(\theta'_t) - g'_i(\theta'_t) \Vert    \nonumber   
    \end{eqnarray}
    where the non-expansiveness of gradient descent operator and Lipschitz smoothness are utilized.

    Combining these two cases for client $i$ and further leveraging Lemma \ref{lmm_grad-bnd-SCAFFOLD-convex}, we obtain
    \begin{eqnarray}
        \mathbb{E}\Vert \theta_{i,k+1} - \theta'_{i,k+1} \Vert &\le& \mathbb{E}\Vert \theta_{i,k} - \theta'_{i,k} \Vert + 2\beta\alpha_{i,k}\mathbb{E}\Vert \theta_t - \theta'_t \Vert + \frac{\alpha_{i,k}}{n_i} \mathbb{E}\Vert g_i(\theta'_{i,k}) - g'_i(\theta'_{i,k}) \Vert \nonumber   \\
        && + \frac{\alpha_{i,k}(1+p_i)}{n_i}\mathbb{E}\Vert g_i(\theta'_t) - g'_i(\theta'_t) \Vert    \nonumber \\
        &\le& \mathbb{E}\Vert \theta_{i,k} - \theta'_{i,k} \Vert + 2\beta\alpha_{i,k}\mathbb{E}\Vert \theta_t - \theta'_t \Vert + \frac{2\alpha_{i,k}}{n_i} \mathbb{E}\Vert g_i(\theta_{i,k})\Vert \nonumber   \\
        && + \frac{2\alpha_{i,k}(1+p_i)}{n_i}\mathbb{E}\Vert g_i(\theta_t) \Vert    \nonumber \\
        &\le& \mathbb{E}\Vert \theta_{i,k} - \theta'_{i,k} \Vert + 2\beta\alpha_{i,k}\mathbb{E}\Vert \theta_t - \theta'_t \Vert + \frac{2\alpha_{i,k}}{n_i}(1+\beta \tilde{\alpha}_{i,t})(\mathbb{E}\Vert \nabla R(\theta_t) \Vert + \sigma)    \nonumber   \\
        && + \frac{2\alpha_{i,k}(1+p_i)}{n_i}(2LD_i + \mathbb{E}\Vert \nabla R(\theta_t) \Vert + \sigma) + \frac{2\alpha_{i,k}}{n_i}2LD_i   \nonumber .
    \end{eqnarray}
    Unrolling it gives
    \begin{eqnarray}    \label{eq_iter-i-SCAFFOLD-convex}
        \mathbb{E}\Vert \theta_{i,K_i} - \theta'_{i,K_i} \Vert &\le& (1 + \beta \tilde{\alpha}_{i,t})\mathbb{E}\Vert \theta_t - \theta'_t \Vert + \frac{2\tilde{\alpha}_{i,t}}{n_i} \big( 2LD_i + (1 + \beta\tilde{\alpha}_{i,t})(\mathbb{E}\Vert \nabla R(\theta_t)\Vert + \sigma) \big)   \nonumber   \\
        && + \frac{2\tilde{\alpha}_{i,t}(1+p_i)}{n_i}(2LD_i + \mathbb{E}\Vert \nabla R(\theta_t) \Vert + \sigma) .
    \end{eqnarray}
    By \eqref{eq_iter-j-SCAFFOLD-convex} and \eqref{eq_iter-i-SCAFFOLD-convex}, we obtain
    \begin{eqnarray}
        \mathbb{E}\Vert \theta_{t+1} - \theta'_{t+1} \Vert &\le& \sum_{j=1}^m p_j \mathbb{E}\Vert \theta_{j,K_j} - \theta'_{j,K_j} \Vert    \nonumber   \\
        &\le& (1 + \beta \hat{\alpha}_t)\mathbb{E}\Vert \theta_t - \theta'_t \Vert + \frac{2 \gamma^1_t}{n} 2LD_i + \frac{2 \gamma^2_t}{n}(\mathbb{E}\Vert \nabla R(\theta_t) \Vert + \sigma),  \nonumber
    \end{eqnarray}
    and we further keep iterate it over $t$ to obtain
    \begin{align*}
        \mathbb{E}\Vert \theta_T - \theta'_T \Vert \le \frac{2}{n} \sum_{t=0}^{T-1} \exp{\Big(2\beta \sum_{l=t+1}^{T-1} \hat{\alpha}_l \Big)} \Big( 2L D_i \gamma^1_{t} + \gamma^2_{t} \mathbb{E}\Vert \nabla R(\theta_{t}) \Vert  + \sigma \gamma^2_{t} \Big)
    \end{align*}
    where we use the fact $1 + x \le e^{x}, \forall x$.
\end{proof}

\subsection{Analysis for FedProx under convex losses}

\begin{lemma}   \label{lmm_drift-FedProx-convex}
    Suppose Assumptions \ref{assump_Lip-continuous}, \ref{assump_bounded-grad-var} and \ref{assump_convexity} hold. Considering FedProx with local update \eqref{eq_FedProx-update}, then for any $\eta_i > 0$, we have for any $i \in [m]$
    $$
        \mathbb{E}\Vert \theta^i_{t+1} - \theta_t \Vert \le \eta_i (\mathbb{E}\Vert \nabla R(\theta_t) \Vert + 2LD_i + \sigma), ~~ \forall t=0,1,\dots.
    $$
\end{lemma}
\begin{proof}
    Recalling the local update \eqref{eq_FedProx-update} of FedProx and according to the first-order optimality condition, we have
    $$
        \eta_i \nabla \hat{R}_{\mathcal{S}_i}(\theta^i_{t+1}) + \theta^i_{t+1} - \theta_t = 0.
    $$
    Moreover, since the function $\eta_i \hat{R}_{\mathcal{S}_i}(\theta) + \frac{1}{2}\Vert \theta - \theta_t \Vert$ is $1$-strongly-convex when Assumption \ref{assump_convexity} holds, we have
    $$
        \Vert \theta^i_{t+1} - \theta_t \Vert \le \Vert \eta_i\nabla \hat{R}_{\mathcal{S}_i}(\theta_{t}) + \theta_{t} - \theta_t \Vert = \eta_i \Vert \nabla \hat{R}_{\mathcal{S}_i}(\theta_t) \Vert
    $$
    by combining the first-order optimality condition. Moreover, note that
    \begin{eqnarray}
        \mathbb{E}\Vert \nabla \hat{R}_{\mathcal{S}_i}(\theta_t) \Vert &\le& \mathbb{E}\Vert \nabla R_i(\theta_t) \Vert + \sigma    \nonumber   \\
        &\le& \mathbb{E}\Vert \nabla R(\theta_t) \Vert + \mathbb{E}\Vert \nabla R_i(\theta_t) - \nabla R(\theta_t) \Vert + \sigma   \nonumber   \\
        &\le& \mathbb{E}\Vert \nabla R(\theta_t) \Vert + 2LD_i + \sigma ,   \nonumber
    \end{eqnarray}
    where we use Lemma \ref{lmm_TV-bound} and note 
    $$
        \mathbb{E}\Vert \nabla \hat{R}_{\mathcal{S}_i}(\theta_t) - \nabla R_i(\theta_t) \Vert \le \frac{1}{n_i}\sum_{j=1}^{n_i} \mathbb{E}\Vert \nabla l(\theta_t;z_{i,j}) - \nabla R_i(\theta_t)\Vert \le \sigma.
    $$
    Thus, we have
    $$
        \mathbb{E}\Vert \theta^i_{t+1} - \theta_t \Vert \le \eta_i (\mathbb{E}\Vert \nabla R(\theta_t) \Vert + 2LD_i + \sigma),
    $$
    which completes the proof.
\end{proof}

\begin{lemma}
    Suppose Assumptions \ref{assump_Lip-continuous}-\ref{assump_convexity} hold and consider FedProx with local update \eqref{eq_FedProx-update}. Then, for any $i \in [m]$ and $j \in [n_i]$, we have
    $$
        \mathbb{E}\Vert \nabla l(\theta^i_{t+1}; z_{i,j}) \Vert \le (1 + \beta \eta_i) (2LD_i + \mathbb{E}\Vert \nabla R(\theta_t) \Vert + \sigma), ~~ \forall t=0,1,\dots.
    $$
\end{lemma}
\begin{proof}
    For any $i \in [m]$ and time $t$,
    \begin{eqnarray}
        \mathbb{E}\Vert \nabla l(\theta^i_{t+1}; z_{i,j}) \Vert &\le& \mathbb{E}\Vert \nabla l(\theta^i_{t+1}; z_{i,j}) - \nabla R_i(\theta^i_{t+1}) \Vert + \mathbb{E}\Vert \nabla R_i(\theta^i_{t+1}) \Vert   \nonumber   \\
        &\le& \mathbb{E}\Vert \nabla R_i(\theta^i_{t+1}) \Vert + \sigma \nonumber   \\
        &\le& \mathbb{E}\Vert \nabla R_i(\theta_t) \Vert + \mathbb{E}\Vert \nabla R_i(\theta^i_{t+1}) - \nabla R_i(\theta_t) \Vert + \sigma \nonumber   \\
        &\le& \beta \mathbb{E}\Vert \theta^i_{t+1} - \theta_t \Vert + \mathbb{E}\Vert \nabla R_i(\theta_t) - \nabla R(\theta_t) \Vert + \mathbb{E}\Vert \nabla R(\theta_t) \Vert + \sigma   \nonumber   \\
        &\le& (1 + \beta \eta_i) (2LD_i + \mathbb{E}\Vert \nabla R(\theta_t) \Vert + \sigma),   \nonumber
    \end{eqnarray}
    where we use Lemma \ref{lmm_TV-bound} and Lemma \ref{lmm_drift-FedProx-convex} in the last step.
\end{proof}

\begin{theorem}[FedProx part of Theorem \ref{thm_gen-convex}]   \label{thm_gen-FedProx-convex}
    Suppose Assumptions \ref{assump_Lip-continuous}-\ref{assump_convexity} hold and consider FedProx (Algorithm \ref{alg_FedProx}). Let $\{ \theta_t \}_{t=0}^{T}$ and $\{ \theta'_t \}_{t=0}^T$ be two trajectories of the server induced by neighboring datasets $\mathcal{S}$ and $\mathcal{S}^{(i)}$, respectively. Suppose $\theta_0 = \theta'_0$. Then,
    \begin{align*}
        \mathbb{E}\Vert \theta_T - \theta'_T \Vert \le \frac{2}{n} \sum_{t=0}^{T-1} \eta_{i}(1 + \beta \eta_{i}) \Big( 2L D_i + \mathbb{E}\Vert \nabla R(\theta_t) \Vert + \sigma \Big) .
    \end{align*}
\end{theorem}
\begin{proof}
    Denoting $\mathrm{prox}_f(x) := \arg\min_{y} f(y) + \frac{1}{2}\Vert y-x \Vert^2$, we can rewrite the local update \eqref{eq_FedProx-update} as
    $$
        \theta^i_{t+1} = \mathrm{prox}_{\eta_i \hat{R}_{\mathcal{S}_i}} (\theta_t).
    $$
    There are two different cases for local updates. For client $j$ with $j\ne i$, we note $\hat{R}_{\mathcal{S}_j}(\cdot) = \hat{R}_{\mathcal{S}'_j}(\cdot)$ in the sense that there is no perturbation for client $j$. In this case, using Lemma \ref{lmm_prox-convex} we obtain
    \begin{eqnarray}
        \Vert \theta^i_{t+1} - (\theta^i_{t+1})' \Vert &=& \Vert \mathrm{prox}_{\eta_i \hat{R}_{\mathcal{S}_i}} (\theta_t) - \mathrm{prox}_{\eta_i \hat{R}_{\mathcal{S}_i}} (\theta'_t) \Vert  \nonumber    \\
        &\le& \Vert \theta_t - \theta'_t \Vert . \nonumber
    \end{eqnarray}

    For client $i$, we note that $\hat{R}_i(\cdot) - \hat{R}'_i(\cdot) = \frac{1}{n_i}( l(\cdot;z_{i,j}) - l(\cdot;z'_{i,j}))$, where $z'_{i,j}$ is the perturbed data point. And we also use $\hat{R}_i$ and $\hat{R}'_i$ to represent $\hat{R}_{\mathcal{S}_i}$ and $\hat{R}_{\mathcal{S}'_i}$ for simplicity. Then, we have
\begin{eqnarray}
    \theta_{t+1}^i &=& \arg\min_{\theta} \eta_i \hat{R}_i(\theta) + \frac{1}{2}\Vert \theta - \theta_t \Vert^2  \nonumber   \\
    (\theta_{t+1}^i)' &=& \arg\min_{\theta} \eta_i \hat{R}'_i(\theta) + \frac{1}{2}\Vert \theta - \theta'_t \Vert^2 . \nonumber
\end{eqnarray}
According to the first-order optimality condition, it yields
\begin{eqnarray}
    \theta_{t+1}^i - \theta_t &=& -\eta_i \nabla \hat{R}_i(\theta_{t+1}^i)  \nonumber   \\
    (\theta_{t+1}^i)' - \theta'_t &=& -\eta_i \nabla \hat{R}'_i((\theta_{t+1}^i)')  \nonumber   \\
    &=& -\eta_i \nabla \hat{R}_i((\theta_{t+1}^i)') + \frac{\eta_i}{n_i}\left( \nabla l((\theta_{t+1}^i)';z_{i,j}) - \nabla l((\theta_{t+1}^i)';z'_{i,j}) \right) . \nonumber
\end{eqnarray}
Moreover, by the monotone property of $\nabla \hat{R}_i(\cdot)$ for convex losses i.e., Lemma \ref{lmm_prox-convex},
\begin{eqnarray}
    \Vert (\theta_{t+1}^i)' - \theta_{t+1}^i \Vert^2 &\le& \langle \theta'_t - \theta_t, (\theta_{t+1}^i)' - \theta_{t+1}^i \rangle - \eta_i \langle \nabla \hat{R}_i((\theta_{t+1}^i)') - \nabla \hat{R}_i(\theta_{t+1}^i), (\theta_{t+1}^i)' - \theta_{t+1}^i \rangle     \nonumber   \\
    && + \frac{\eta_i}{n_i} \langle (\theta_{t+1}^i)' - \theta_{t+1}^i, \nabla l((\theta_{t+1}^i)'; z_{i,j}) - \nabla l(\theta_{t+1}^i; z'_{i,j}) \rangle   \nonumber   \\
    &\le& \langle \theta'_t - \theta_t, (\theta_{t+1}^i)' - \theta_{t+1}^i \rangle + \frac{\eta_i}{n_i} \langle (\theta_{t+1}^i)' - \theta_{t+1}^i, \nabla l((\theta_{t+1}^i)'; z_{i,j}) - \nabla l(\theta_{t+1}^i; z'_{i,j}) \rangle   \nonumber
\end{eqnarray}
which further implies by symmetry of $z_{i,j}$ and $z'_{i,j}$,
\begin{eqnarray}
    \Vert (\theta_{t+1}^i)' - \theta_{t+1}^i \Vert \le \Vert \theta'_t - \theta_t \Vert + \frac{\eta_i}{n_i} \Vert \nabla l(\theta_{t+1}^i; z_{i,j}) - \nabla l(\theta_{t+1}^i; z'_{i,j}) \Vert .   \nonumber
\end{eqnarray}

Combining two cases gives
\begin{eqnarray}
    \mathbb{E}\Vert \theta_{t+1} - \theta'_{t+1} \Vert &\le& \sum_{j=1}^m p_j \mathbb{E}\Vert \theta^j_{t+1} - (\theta^j_{t+1})' \Vert  \nonumber  \\
    &\le& \mathbb{E}\Vert \theta_t - \theta'_t \Vert + \frac{\eta_i}{n} \mathbb{E}\Vert \nabla l(\theta_{t+1}^i; z_{i,j}) - \nabla l(\theta_{t+1}^i; z'_{i,j}) \Vert , \nonumber  \\
    &\le& \mathbb{E}\Vert \theta_t - \theta'_t \Vert + \frac{2\eta_i}{n} \mathbb{E}\Vert \nabla l(\theta_{t+1}^i; z_{i,j})\Vert \nonumber   \\
    &\le& \mathbb{E}\Vert \theta_t - \theta'_t \Vert + \frac{2\eta_i}{n} (1 + \beta \eta_i) (2LD_i + \mathbb{E}\Vert \nabla R(\theta_t) \Vert + \sigma) . \nonumber
\end{eqnarray}
Unrolling it over $t$ completes the proof.
\end{proof}

\subsection{Proof of Theorem \ref{thm_convergence}}
Our results are established based on the following convergence results of three algorithms, which are formaly shown in Theorem \ref{thm_con-formal}. These results are based on the following assumptions.
\begin{assumption}  \label{assmp_dissimilarity}
    There exist constants $G \ge 0$ and $B \ge 1$ such that 
    $$
        \sum_{i=1}^m p_i \Vert \nabla R_i(\theta) \Vert^2 \le 2 G^2 + B^2 \Vert \nabla R(\theta) \Vert^2, ~~ \forall \theta.
    $$
\end{assumption}

\begin{assumption}  \label{assump_dissimilarity-strong}
    There exist constants $G_i \ge 0$ such that for any $i \in [m]$, 
    $$
         \Vert \nabla R_i(\theta) - \nabla R(\theta) \Vert \le G_i, ~~ \forall \theta.
    $$
\end{assumption}

In fact, Assumption \ref{assump_dissimilarity-strong} is a stronger assumption compared to Assumption \ref{assmp_dissimilarity}, which is shown by the following proposition.
\begin{proposition} \label{prop_dissimilarity}
    Assumption \ref{assump_dissimilarity-strong} implies Assumption \ref{assmp_dissimilarity}.
\end{proposition}
\begin{proof}
    Note that given Assumption \ref{assump_dissimilarity-strong}
    \begin{eqnarray}
        \Vert \nabla R_i(\theta) \Vert \le \Vert \nabla R(\theta)\Vert + \Vert \nabla R_i(\theta) - \nabla R(\theta) \Vert \le G_i + \Vert \nabla R(\theta) \Vert, \nonumber
    \end{eqnarray}
    which implies
    $$
        \Vert \nabla R_i(\theta)\Vert^2 \le 2G_i^2 + 2\Vert \nabla R(\theta)\Vert^2 .
    $$
    Taking the weighted sum of $p_i$ and we conclude $G^2 = \sum_{i=1}^m p_i G_i^2$, $B^2 = 2$.
\end{proof}

In the next proposition, we characterize $G_i$ defined in Assumption \ref{assump_dissimilarity-strong} by directly usting Lemma \ref{lmm_TV-bound}.
\begin{proposition} \label{prop_Gi}
    Suppose Assumption \ref{assump_Lip-continuous} holds. Then, $G_i = 2Ld_{TV}(P_i, P)$ defined in Assumption \ref{assump_dissimilarity-strong}.
\end{proposition}

Then, we state the existing convergence results for FedAvg, SCAFFOLD and FedProx in the following theorem.
\begin{theorem} \cite{SCAFFOLD,FedProx} \label{thm_con-formal}
    Suppose Assumption \ref{assump_bounded-grad-var} holds and $K_i = K, \forall i\in [m]$. 
    
    For FedAvg (Algorithm \ref{alg_FedAvg}) with Assumptions \ref{assump_Lip-smooth},\ref{assmp_dissimilarity} satisfied and $\alpha_{i,k} \le \frac{1}{(1+B^2)8\beta K}$, we have 
    \begin{eqnarray}    \label{eq_con-FedAvg}
        \frac{1}{T}\sum_{t=0}^{T-1}\mathbb{E}\Vert \nabla R(\theta_{t}) \Vert^2 \le \mathcal{O}\left( \frac{\sqrt{\Delta_0}}{\sqrt{TKm}} + \frac{(\Delta_0 G)^{2/3}}{T^{2/3}} + \frac{B^2 \Delta_0}{T} \right) .
    \end{eqnarray}

    For SCAFFOLD (Algorithm \ref{alg_SCAFFOLD}) with Assumption \ref{assump_Lip-smooth} and $\alpha_{i,k} \le \frac{1}{24\beta K}$, we have
    \begin{eqnarray}    \label{eq_con-SCAFFOLD}
        \frac{1}{T}\sum_{t=0}^{T-1}\mathbb{E}\Vert \nabla R(\theta_{t}) \Vert^2 \le \mathcal{O}\left( \frac{\sqrt{\Delta_0}}{\sqrt{TKm}} + \frac{\Delta_0}{T} \right) .
    \end{eqnarray}

    Suppose Assumption \ref{assump_dissimilarity-strong} hold. For FedProx (Algorithm \ref{alg_FedProx}) with eigenvalues of $\nabla^2 R(\theta)$ lower bounded and $\eta_i$ chosen small enough, we have
    \begin{eqnarray}    \label{eq_con-FedProx}
        \frac{1}{T}\sum_{t=0}^{T-1}\mathbb{E}\Vert \nabla R(\theta_{t}) \Vert^2 \le \mathcal{O}\left( \frac{\Delta_0 \sum_{i=1}^m p_i G_i^2}{\sqrt{T}} + \frac{\Delta_0}{T} \right) ,
    \end{eqnarray}
    where $\Delta_0 := \mathbb{E}[R(\theta_0) - R(\theta^*)]$.
\end{theorem}

\paragraph{Proof of FedAvg and FedProx parts of Theorem \ref{thm_convergence}.}
It follows the fact that stepsizes $\alpha_{i,k}$ and $\eta_i$ are upper bounded by some constant $c$ and 
\begin{equation}    \label{eq_Jensen}
    \left(\sum_{t=0}^{T-1} c \mathbb{E}\Vert \nabla R(\theta_t) \Vert \right)^2 \le T \sum_{t=0}^{T-1} c^2 \big(\mathbb{E}\Vert \nabla R(\theta_t) \Vert\big)^2 \le T \sum_{t=0}^{T-1} c^2 \mathbb{E}\Vert \nabla R(\theta_t) \Vert^2,
\end{equation}
where the second inequality follows Jensen's inequality. Combining Propositions \ref{prop_dissimilarity} and \ref{prop_Gi} with \eqref{eq_con-FedAvg},\eqref{eq_con-FedProx} completes the proof.

\paragraph{Proof of SCAFFOLD part of Theorem \ref{thm_convergence}.} 
To get the result for SCAFFOLD in Theorem \ref{thm_convergence}, we further note that $\gamma^2_t$ is upper bounded by some constant $\bar{\gamma}$ and when $\alpha_{i,k} \le 1/[24\beta K(t+1)]$
\begin{eqnarray}    \label{eq_bound-exp}
    \sum_{t=0}^{T-1}\exp{\left( 2\beta \sum_{l=t+1}^{T-1} \hat{\alpha}_l \right)} \gamma^2_t \mathbb{E}\Vert \nabla R(\theta_t) \Vert &\le& \sum_{t=0}^{T-1}\exp{\left( \frac{1}{12} \log(T) \right)} \gamma^2_t \mathbb{E}\Vert \nabla R(\theta_t) \Vert   \nonumber \\
    &\le& T^{1/12} \sum_{t=0}^{T-1} \bar{\gamma} \mathbb{E}\Vert \nabla R(\theta_t) \Vert .
\end{eqnarray}
Combining \eqref{eq_bound-exp} with \eqref{eq_con-SCAFFOLD} and \eqref{eq_Jensen} completes the proof.

\subsection{Proofs of Corollaries \ref{coro_gen-convex} and \ref{coro_gen-iid}}

To obtain Corollary \ref{coro_gen-convex}, we note that under Assumption \ref{assump_Lip-continuous}, 
$$
    \mathbb{E}_{\mathcal{A,S},z'_{i,j}}|l(\theta_T; z'_{i,j}) - l(\theta'_T;z'_{i,j})| \le L \mathbb{E}\Vert \theta_T - \theta'_T \Vert, ~ \forall j\in[n_i]
$$
and then combining Theorems \ref{thm_stability-gen},\ref{thm_gen-convex},\ref{thm_convergence} provides the results.

To obtain Corollary \ref{coro_gen-iid}, we start from Theorem \ref{thm_gen-convex}. Note that given Assumption \ref{assump_Lip-continuous}, we can bound $\mathbb{E}\Vert \nabla R(\theta_t)\Vert$ by Lipschitz constant $L$, i.e., $\mathbb{E}\Vert \nabla R(\theta_t) \Vert \le L$. Moreover, under the i.i.d. case, meaning $D_i = 0, \forall i \in [m]$, we conclude the proof by using the same techniques as those in \eqref{eq_Jensen} and \eqref{eq_bound-exp}.

\begin{remark}
    Note that the bounds in Corollary \ref{coro_gen-iid} are also looser, compared to those in Corollary \ref{coro_gen-convex} even when $D_{max}=0$ (which corresponds to the i.i.d. case). To see this, note that bounds in Corollary \ref{coro_gen-iid} are linear in $T$, while bounds in Corollary \ref{coro_gen-convex} are with $\mathcal{O}(T^{q})$ for some $q < 1$. Moreover, more information is captured in Corollary \ref{coro_gen-convex}, e.g., number of clients $m$, distance of the initial point to the optimal one $\Delta_0$, etc.
\end{remark}

\section{Generalization Bounds for Non-convex Losses}   \label{apx_proof-nonconvex}

\subsection{Analysis for FedAvg under non-convex losses}

\begin{lemma}   \label{lmm_drift-FedAvg-nonconvex}
    Suppose Assumptions \ref{assump_Lip-continuous}-\ref{assump_Lip-smooth} hold. Then for FedAvg with $\alpha_{i,k} \le c/\beta$ for some $c > 0$,
    $$
        \mathbb{E}\Vert \theta_{i,k} - \theta_t \Vert \le (1+c)^{K_i-1}\tilde{\alpha}_{i,t} \big( \mathbb{E}\Vert \nabla R(\theta_t) \Vert + 2LD_i + \sigma \big), ~\forall k=1,\dots,K_i ,
    $$
    where $\tilde{\alpha}_{i,t} = \sum_{k=0}^{K_i-1} \alpha_{i,k}$.
\end{lemma}
\begin{proof}
    Considering local update \eqref{eq_FedAvg-update} of FedAvg
    \begin{eqnarray}
        \mathbb{E}\Vert \theta_{i,k+1} - \theta_t \Vert &=& \mathbb{E} \Vert \theta_{i,k} - \alpha_{i,k} g_i(\theta_{i,k}) - \theta_t \Vert   \nonumber   \\
        &\le& \mathbb{E}\Vert \theta_{i,k} - \theta_t - \alpha_{i,k} (g_i(\theta_{i,k}) - g_i(\theta_t))  \Vert + \alpha_{i,k}\mathbb{E}\Vert g_i(\theta_t) \Vert   \nonumber   \\
        &\le& (1 + \beta \alpha_{i,k})\mathbb{E}\Vert \theta_{i,k} - \theta_t \Vert + \alpha_{i,k}\mathbb{E}\Vert g_i(\theta_t) \Vert    \nonumber   \\
        &\le& (1 + \beta \alpha_{i,k})\mathbb{E}\Vert \theta_{i,k} - \theta_t \Vert + \alpha_{i,k}(\mathbb{E}\Vert g_i(\theta_t) - \nabla R_i(\theta_t) \Vert + \mathbb{E}\Vert \nabla R_i(\theta_t) \Vert)   \nonumber   \\
        &\le& (1 + \beta \alpha_{i,k})\mathbb{E}\Vert \theta_{i,k} - \theta_t \Vert + \alpha_{i,k}(\mathbb{E}\Vert \nabla R_i(\theta_t) \Vert + \sigma),    \nonumber
    \end{eqnarray}
    where we use Assumptions \ref{assump_bounded-grad-var} and \ref{assump_Lip-smooth}. Unrolling the above and noting $\theta_{i,0} = \theta_t$ yields
    \begin{eqnarray}
        \mathbb{E}\Vert \theta_{i,k} - \theta_t \Vert
        &\le& \sum_{l=0}^{k-1} \alpha_{i,l} \big( \mathbb{E}\Vert \nabla R_i(\theta_t) \Vert + \sigma \big) (1 + c)^{k-1-l}  \nonumber   \\
        &\le& \sum_{l=0}^{K_i-1} \alpha_{i,l} \big( \mathbb{E}\Vert \nabla R_i(\theta_t) \Vert + \sigma \big) (1 + c)^{K_i-1}  \nonumber  \\
        &\le& (1+c)^{K_i - 1}\tilde{\alpha}_{i,t} \big( \mathbb{E}\Vert \nabla R(\theta_t) \Vert + 2LD_i + \sigma \big),   \nonumber
    \end{eqnarray}
    where the last inequality follows Lemma \ref{lmm_TV-bound}.
\end{proof}

\begin{lemma}   \label{lmm_grad-bnd-FedAvg-nonconvex}
    Given Assumptions \ref{assump_Lip-continuous}-\ref{assump_Lip-smooth} and considering \eqref{eq_FedAvg-update} of FedAvg, for $\alpha_{i,k} \le c/\beta$ with some $c>0$, we have
    $$
        \mathbb{E}\Vert g_i(\theta_{i,k}) \Vert \le (1 + (1+c)^{K_i-1}\beta \tilde{\alpha}_{i,t}) \big( \mathbb{E}\Vert \nabla R(\theta_t) \Vert + 2LD_i + \sigma \big),
    $$
    where $g_i(\cdot)$ is the sampled gradient of client $i$, $\tilde{\alpha}_{i,t} = \sum_{k=0}^{K_i-1} \alpha_{i,k}$.
\end{lemma}
\begin{proof}
    Using Lemmas \ref{lmm_TV-bound} and \ref{lmm_drift-FedAvg-nonconvex}, we obtain
    \begin{eqnarray}
        \mathbb{E} \Vert g_i(\theta_{i,k}) \Vert &\le& \mathbb{E} \Vert g_i(\theta_{i,k}) - \nabla R_i(\theta_{i,k}) \Vert + \mathbb{E} \Vert \nabla R_i(\theta_{i,k}) \Vert  \nonumber   \\
        &\le& \mathbb{E}\Vert \nabla R_i(\theta_{i,k}) \Vert + \sigma   \nonumber   \\
        &\le& \mathbb{E}\Vert \nabla R_i(\theta_t) \Vert + \mathbb{E}\Vert \nabla R_i(\theta_{i,k}) - \nabla R_i(\theta_t) \Vert + \sigma   \nonumber   \\
        &\le& \mathbb{E}\Vert \nabla R(\theta_t) \Vert + \mathbb{E}\Vert \nabla R_i(\theta_t) - \nabla R(\theta_t) \Vert + \beta \mathbb{E}\Vert \theta_{i,k} - \theta_t \Vert + \sigma \nonumber   \\
        &\le& (1 + (1+c)^{K_i-1}\beta \tilde{\alpha}_i) \big( \mathbb{E}\Vert \nabla R(\theta_t) \Vert + 2LD_i + \sigma \big) . \nonumber
    \end{eqnarray}
\end{proof}

\begin{theorem}[FedAvg part of Theorem \ref{thm_gen-nonconvex}] \label{thm_gen-FedAvg-nonconvex}
    Suppose Assumptions \ref{assump_Lip-continuous}-\ref{assump_Lip-smooth} hold and consider FedAvg (Algorithm \ref{alg_FedAvg}). Let $K_i = K, \forall i \in [m]$ and $\alpha_{i,k} \le \frac{1}{24\beta K(t+1)}$. Then, 
    \begin{align*}    
        \epsilon_{gen} \le \mathcal{O}\Big( \frac{T^{\frac{1}{24}} \log{T}}{n} (D_{max}+\sigma) \Big) + \mathcal{O} \Big( \big( \frac{\Delta_0}{K m} \big)^{\frac{1}{4}} \frac{T^{\frac{5}{6}}}{n} + \big( \Delta_0^2 \tilde{D} \big)^{\frac{1}{6}} \frac{T^{\frac{3}{4}}}{n} + \sqrt{\Delta_0} \frac{T^{\frac{7}{12}}}{n} \Big) .
    \end{align*}
\end{theorem}
\begin{proof}
    The proof is similar to that of Theorem \ref{thm_gen-FedAvg-convex}. Given time index $t$ and for client $j$ with $j \ne i$, we have
    \begin{eqnarray}
        \mathbb{E}\Vert \theta_{j,k+1} - \theta'_{j,k+1}\Vert &=& \mathbb{E}\Vert \theta_{j,k} - \theta'_{j,k} - \alpha_{j,k}(g_j(\theta_{j,k}) - g_j(\theta'_{j,k})) \Vert   \nonumber   \\
        &\le& (1 + \beta \alpha_{j,k})\mathbb{E}\Vert \theta_{j,k} - \theta'_{j,k} \Vert .   \nonumber
    \end{eqnarray}
    And unrolling it gives
    \begin{eqnarray}    \label{eq_iter-j-FedAvg-nonconvex}
        \mathbb{E}\Vert \theta_{j,K_j} - \theta'_{j,K_j} \Vert &\le& \prod_{k=0}^{K_j-1} (1 + \beta\alpha_{j,k})\mathbb{E}\Vert \theta_t - \theta'_t\Vert  \nonumber \\
        &\le& e^{\beta \tilde{\alpha}_{j,t}} \mathbb{E}\Vert \theta_t - \theta'_t \Vert, ~~ \forall j \ne i,
    \end{eqnarray}
    where we use $1 + x \le e^x, \forall x$.

    For client $i$, there are two cases to consider. In the first case, SGD selects non-perturbed samples in $\mathcal{S}$ and $\mathcal{S}^{(i)}$, which happens with probability $1 - 1/n_i$. Then, we have
    $$
        \Vert \theta_{i,k+1} - \theta'_{i,k+1} \Vert \le (1 + \beta \alpha_{i,k}) \Vert \theta_{i,k} - \theta'_{i,k} \Vert.
    $$

    In the second case, SGD encounters the perturbed sample at time step $k$, which happens with probability $1/n_i$. Then, we have
    \begin{eqnarray}
        \Vert \theta_{i,k+1} - \theta'_{i,k+1} \Vert &=& \Vert \theta_{i,k} - \theta'_{i,k} - \alpha_{i,k}(g_i(\theta_{i,k}) - g'_i(\theta'_{i,k})) \Vert   \nonumber   \\
        &\le& \Vert \theta_{i,k} - \theta'_{i,k} - \alpha_{i,k}(g_i(\theta_{i,k}) - g_i(\theta'_{i,k})) \Vert + \alpha_{i,k}\Vert g_i(\theta'_{i,k}) - g'_i(\theta'_{i,k}) \Vert    \nonumber   \\
        &\le& (1 + \beta \alpha_{i,k})\Vert \theta_{i,k} - \theta'_{i,k} \Vert + \alpha_{i,k}\Vert g_i(\theta'_{i,k}) - g'_i(\theta'_{i,k}) \Vert .  \nonumber
    \end{eqnarray}
    Combining these two cases for client $i$ we have
    \begin{eqnarray}
        \mathbb{E}\Vert \theta_{i,k+1} - \theta'_{i,k+1} \Vert &\le& (1 + \beta \alpha_{i,k})\mathbb{E}\Vert \theta_{i,k} - \theta'_{i,k} \Vert + \frac{\alpha_{i,k}}{n_i}\mathbb{E}\Vert g_i(\theta'_{i,k}) - g'_i(\theta'_{i,k}) \Vert    \nonumber   \\
        &\le& (1 + \beta \alpha_{i,k})\mathbb{E}\Vert \theta_{i,k} - \theta'_{i,k} \Vert + \frac{\alpha_{i,k}}{n_i}\mathbb{E}\Vert g_i(\theta_{i,k}) \Vert,  \nonumber  \\
        &\le& (1 + \beta \alpha_{i,k})\mathbb{E}\Vert \theta_{i,k} - \theta'_{i,k} \Vert + \frac{2\alpha_{i,k}}{n_i}(1 + (1+c)^{K_i - 1}\beta\tilde{\alpha}_{i,t})\big(\sigma   \nonumber   \\
        && + \mathbb{E}\Vert \nabla R(\theta_t) \Vert + 2LD_i\big)  \nonumber   \\
        &\le& (1 + \beta \alpha_{i,k})\mathbb{E}\Vert \theta_{i,k} - \theta'_{i,k} \Vert + \frac{2\alpha_{i,k} \tilde{c}}{n_i}(\mathbb{E}\Vert \nabla R(\theta_t) \Vert + 2LD_i + \sigma)  \nonumber
    \end{eqnarray}
    where we use Lemma \ref{lmm_grad-bnd-FedAvg-nonconvex} and we let $\tilde{c}$ be an upper bound of $1 + (1+c)^{K_i - 1}\beta\tilde{\alpha}_{i,t}$ since $\tilde{\alpha}_{i,t}$ is bounded above.
    Then unrolling it gives
    \begin{eqnarray}    \label{eq_iter-i-FedAvg-nonconvex}
        \mathbb{E}\Vert \theta_{i,K_i} - \theta'_{i,K_i} \Vert &\le& \prod_{k=0}^{K_i-1}(1+\beta\alpha_{i,k})\mathbb{E}\Vert \theta_t - \theta'_t \Vert + \bigg(\frac{2}{n_i}\sum_{k=0}^{K_i-1}\alpha_{i,k} \tilde{c} \prod_{l=k+1}^{K_i-1} (1+\beta\alpha_{i,l})  \nonumber \\
        && \cdot (\mathbb{E}\Vert \nabla R(\theta_t) \Vert + 2LD_i + \sigma) \bigg) \nonumber   \\
        &\le& e^{\beta \tilde{\alpha}_{i,t}}\mathbb{E}\Vert \theta_t - \theta'_t \Vert + \frac{2}{n_i} \tilde{c} \tilde{\alpha}_{i,t} e^{\beta \tilde{\alpha}_{i,t}}(\mathbb{E}\Vert \nabla R(\theta_t) \Vert + 2LD_i + \sigma) .
    \end{eqnarray}
    By \eqref{eq_iter-j-FedAvg-nonconvex} and \eqref{eq_iter-i-FedAvg-nonconvex} we have
    \begin{eqnarray}
        \mathbb{E}\Vert \theta_{t+1} - \theta'_{t+1} \Vert &\le& \sum_{i=1}^m p_i \mathbb{E}\Vert \theta_{i,K_i} - \theta'_{i,K_i} \Vert  \nonumber \\
        &\le& e^{\beta \tilde{\alpha}_{i,t}}\mathbb{E}\Vert \theta_t - \theta'_t \Vert + \frac{2}{n} \tilde{c} \tilde{\alpha}_{i,t} e^{\beta \tilde{\alpha}_{i,t}}(\mathbb{E}\Vert \nabla R(\theta_t) \Vert + 2LD_i + \sigma)  \nonumber
    \end{eqnarray}
    where we also use $p_i = n_i/n$ in the last step. Further, unrolling the above over $t$ and noting $\theta_0 = \theta'_0$, we obtain
    \begin{eqnarray}
        \mathbb{E}\Vert \theta_T - \theta'_T \Vert \le \frac{2\tilde{c}}{n}\sum_{t=0}^{T-1} \exp\left( \beta \sum_{l=t+1}^{T-1} \tilde{\alpha}_{i,t} \right) \tilde{\alpha}_{i,t} e^{\beta \tilde{\alpha}_{i,t}}(\mathbb{E}\Vert \nabla R(\theta_t) \Vert + 2LD_i + \sigma)  \nonumber
    \end{eqnarray}
    When the diminishing stepsizes are chosen in the statement of the theorem, we further combine Theorem \ref{thm_stability-gen} and the same techniques used in Theorem \ref{thm_convergence}, we conclude the proof.
\end{proof}

\subsection{Analysis for SCAFFOLD under non-convex losses}

\begin{lemma}   \label{lmm_drift-SCAFFOD-nonconvex}
    Suppose Assumptions \ref{assump_Lip-continuous}-\ref{assump_Lip-smooth} hold. Running SCAFFOLD with $\alpha_{i,k} \le c/\beta$ for some $c>0$, then for any $i \in [m]$
    $$
        \mathbb{E}\Vert \theta_{i,k} - \theta_t \Vert \le (1+c)^{K_i-1}\tilde{\alpha}_{i,t} (\mathbb{E}\Vert R(\theta_t)\Vert + \sigma), ~~ \forall k=1,\dots,K_i 
    $$
    where $\tilde{\alpha}_{i,t}=\sum_{k=0}^{K_i-1}\alpha_{i,k}$.
\end{lemma}
\begin{proof}
    Considering local update \eqref{eq_SCAFFOLD-update} of SCAFFOLD
    \begin{eqnarray}
        \mathbb{E}\Vert \theta_{i,k+1} - \theta_t \Vert &=& \mathbb{E}\Vert \theta_{i,k} - \alpha_{i,k}(g_i(\theta_{i,k}) - g_i(\theta_t) + g(\theta_t)) - \theta_t \Vert \nonumber   \\
        &\le& \mathbb{E}\Vert \theta_{i,k} - \theta_t - \alpha_{i,k}(g_i(\theta_{i,k}) - g_i(\theta_t)) \Vert + \alpha_{i,k}\mathbb{E}\Vert g(\theta_t) \Vert   \nonumber   \\
        &\le& (1 + \beta \alpha_{i,k})\mathbb{E}\Vert \theta_{i,k} - \theta_t \Vert + \alpha_{i,k}\mathbb{E}\Vert g(\theta_t) \Vert \nonumber   \\
        &\le& (1 + \beta \alpha_{i,k})\mathbb{E}\Vert \theta_{i,k} - \theta_t \Vert + \alpha_{i,k}(\mathbb{E}\Vert R(\theta_t) \Vert + \sigma ) \nonumber
    \end{eqnarray}
    where we use Assumptions \ref{assump_bounded-grad-var} and \ref{assump_Lip-smooth}. Therefore, for any $k=1,\dots,K_i-1$,
    \begin{eqnarray}
        \mathbb{E}\Vert \theta_{i,k} - \theta_t \Vert &\le& \sum_{k=0}^{K_i-1} \alpha_{i,k} (\mathbb{E}\Vert R(\theta_t) \Vert + \sigma ) (1 + c)^{K_i - 1}  \nonumber  \\
        &=& \tilde{\alpha}_{i,t}(1+c)^{K_i-1}(\mathbb{E}\Vert R(\theta_t) \Vert + \sigma ) \nonumber
    \end{eqnarray}
    which completes the proof.
\end{proof}

\begin{lemma}   \label{lmm_grad-bnd-SCAFFOLD-nonconvex}
    Given Assumptions \ref{assump_Lip-continuous}-\ref{assump_Lip-smooth} and considering SCAFFOLD (Algorithm \ref{alg_SCAFFOLD}), with $\alpha_{i,k}\le c/\beta$ for some $c>0$ we have the following inequalities
    \begin{eqnarray}
        \mathbb{E}\Vert g_i(\theta_{i,k}) \Vert &\le& (1+\beta \tilde{\alpha}_{i,t}(1+c)^{K_i-1})(\mathbb{E}\Vert \nabla R(\theta_t) \Vert + \sigma) + 2L D_i,   \nonumber  \\
        \mathbb{E}\Vert g_i(\theta_t) \Vert &\le& 2LD_i + \mathbb{E}\Vert \nabla R(\theta_t) \Vert + \sigma   \nonumber
    \end{eqnarray}
    for any $i \in [m]$, $k=0,\dots,K_i-1$ and $t=0,1,\dots$.
\end{lemma}
\begin{proof}
    Note that based on Assumption \ref{assump_bounded-grad-var},
    \begin{eqnarray}
        \mathbb{E}\Vert g_i(\theta_{i,k}) \Vert &\le& \mathbb{E}\Vert \nabla R_i(\theta_{i,k}) \Vert + \sigma   \nonumber   \\
        &\le& \mathbb{E}\Vert \nabla R_i(\theta_{i,k}) - \nabla R_i(\theta_t) \Vert + \mathbb{E}\Vert \nabla R_i(\theta_t) \Vert + \sigma   \nonumber   \\
        &\le& \beta \mathbb{E}\Vert \theta_{i,k} - \theta_t \Vert + \mathbb{E}\Vert \nabla R_i(\theta_t) - \nabla R(\theta_t) \Vert + \mathbb{E}\Vert \nabla R(\theta_t) \Vert + \sigma \nonumber   \\
        &\le& (1+\beta \tilde{\alpha}_{i,t}(1+c)^{K_i-1})(\mathbb{E}\Vert \nabla R(\theta_t) \Vert + \sigma) + 2L D_i,   \nonumber
    \end{eqnarray}
    where we use Lemmas \ref{lmm_TV-bound} and \ref{lmm_drift-SCAFFOD-nonconvex}.

    Similarly, using same techniques we have
    \begin{eqnarray}
        \mathbb{E}\Vert g_i(\theta_t) \Vert &\le& \mathbb{E}\Vert \nabla R_i(\theta_t) \Vert + \sigma   \nonumber   \\
        &\le& \mathbb{E}\Vert \nabla R_i(\theta_t) - \nabla R(\theta_t) \Vert + \mathbb{E}\Vert \nabla R(\theta_t) \Vert + \sigma   \nonumber   \\
        &\le& 2LD_i + \mathbb{E}\Vert \nabla R(\theta_t) \Vert + \sigma .  \nonumber
    \end{eqnarray}
\end{proof}

\begin{theorem}[SCAFFOLD part of Theorem \ref{thm_gen-nonconvex}]
    Suppose Assumptions \ref{assump_Lip-continuous}-\ref{assump_Lip-smooth} hold and consider SCAFFOLD (Algorithm \ref{alg_SCAFFOLD}). Let $K_i = K$ and $\alpha_{i,k} \le \frac{1}{24\beta K(t+1)}$, $ \forall i \in [m]$
    \begin{align*}    
        \epsilon_{gen} \le \mathcal{O}\Big( \frac{T^{\frac{1}{8}} \log{T}}{n}  D_{max} \Big) + \mathcal{O} \Big( \big( \frac{\Delta_0}{K m} \Big)^{\frac{1}{4}} \frac{T^{\frac{7}{8}}}{n} + \sqrt{\Delta_0} \frac{T^{\frac{5}{8}}}{n} \big) + \mathcal{O} \Big( \frac{T^{\frac{1}{8}} (\log{T} + 1)}{n} \sigma \Big),
    \end{align*}
    where $\Delta_0 = \mathbb{E}[R(\theta_0) - R(\theta^*)]$.
\end{theorem}
\begin{proof}
    Similar to the proof of Theorem \ref{thm_gen-SCAFFOLD-convex}, considering client $j$ with $j\ne i$, there are two cases. In the first case, SCAFFOLD does not select the perturbed sample from client $i$'s dataset at local step $k$. Then, with probability equal to $1 - 1/n_i$, 
    \begin{eqnarray}
        \Vert \theta_{j,k+1} - \theta'_{j,k+1} \Vert &\le& \Vert \theta_{j,k} - \theta'_{j,k} - \alpha_{j,k}(g_j(\theta_{j,k}) - g_j(\theta'_{j,k}))\Vert + \alpha_{j,k}\Vert g_j(\theta_t) - g_j(\theta'_t) \Vert  \nonumber   \\
        && + \alpha_{j,k}\Vert g(\theta_t) - g(\theta'_t)\Vert  \nonumber   \\
        &\le& (1+\beta\alpha_{j,k})\Vert \theta_{j,k} - \theta'_{j,k} \Vert + 2\alpha_{j,k}\beta \Vert \theta_t - \theta'_t \Vert    \nonumber   
    \end{eqnarray}
    where the second inequality follows Assumption \ref{assump_Lip-smooth}.

    In the second case, there is with probability $1/n_i$ that the perturbed sample is selected during the local update of step $k$. Then,
    \begin{eqnarray}    
        \Vert \theta_{j,k+1} - \theta'_{j,k+1} \Vert &\le& \Vert \theta_{j,k} - \theta'_{j,k} - \alpha_{j,k}(g_j(\theta_{j,k}) - g_j(\theta'_{j,k}))\Vert + \alpha_{j,k}\Vert g_j(\theta_t) - g_j(\theta'_t) \Vert  \nonumber   \\
        && + \alpha_{j,k}\Vert g(\theta_t) - g'(\theta'_t)\Vert  \nonumber   \\
        &\le& \Vert \theta_{j,k} - \theta'_{j,k} - \alpha_{j,k}(g_j(\theta_{j,k}) - g_j(\theta'_{j,k}))\Vert + \alpha_{j,k}\Vert g_j(\theta_t) - g_j(\theta'_t) \Vert  \nonumber   \\
        && + \alpha_{j,k}\Vert g(\theta_t) - g(\theta'_t)\Vert + \alpha_{j,k}\Vert g(\theta'_t) - g'(\theta'_t) \Vert  \nonumber   \\
        &\le& (1+\beta\alpha_{i,k})\Vert \theta_{j,k} - \theta'_{j,k} \Vert + 2\beta \alpha_{j,k}\Vert \theta_t - \theta'_t \Vert + \alpha_{j,k}p_i\Vert g_i(\theta'_t) - g'_i(\theta'_t) \Vert .   \nonumber
    \end{eqnarray}
    We again use Assumption \ref{assump_Lip-smooth} in the last step. Combining two cases, we have for client $j$ with $j\ne i$
    \begin{eqnarray}
        \mathbb{E}\Vert \theta_{j,k+1} - \theta'_{j,k+1} \Vert 
        &\le& (1+\beta\alpha_{j,k})\mathbb{E}\Vert \theta_{j,k} - \theta'_{j,k} \Vert + 2\beta\alpha_{j,k}\mathbb{E}\Vert \theta_t - \theta'_t \Vert + \frac{2\alpha_{j,k}}{n}\mathbb{E}\Vert g_i(\theta_t) \Vert.   \nonumber  
    \end{eqnarray}
    Unrolling it over $k$ we obtain
    \begin{eqnarray}    \label{eq_iter-j-SCAFFOLD-nonconvex}
        \mathbb{E}\Vert \theta_{j,K_j} - \theta'_{j,K_j} \Vert &\le& \prod_{k=0}^{K_j-1}(1+\beta\alpha_{j,k})\mathbb{E}\Vert \theta_t - \theta'_t \Vert + \bigg(\sum_{k=0}^{K_j-1} \big(\prod_{l=k+1}^{K_j-1}(1 + \beta\alpha_{j,l}) \big) \nonumber  \\
        && \cdot (2\beta\alpha_{j,k}\mathbb{E}\Vert \theta_t - \theta'_t \Vert + \frac{2\alpha_{j,k}}{n}\mathbb{E}\Vert g_i(\theta_t) \Vert) \bigg) \nonumber   \\
        &\le& (1+\beta\tilde{\alpha}_{j,t})e^{\beta\tilde{\alpha}_{j,t}} \mathbb{E}\Vert \theta_t - \theta'_t \Vert + \frac{2\tilde{\alpha}_{j,t}}{n}e^{\beta\tilde{\alpha}_{j,t}}\mathbb{E}\Vert g_i(\theta_t) \Vert
    \end{eqnarray}
    where we use the fact $1 + x \le e^x$ and Lemma \ref{lmm_grad-bnd-SCAFFOLD-nonconvex} in the last step.

    For client $i$, there are two cases as well. In the first case, the perturbed sample is not selected at step $k$, which happens with probability $1 - 1/n_i$. Then,
    \begin{eqnarray}
        \Vert \theta_{i,k+1} - \theta'_{i,k+1} \Vert &\le& \Vert \theta_{i,k} - \theta'_{i,k} - \alpha_{i,k}(g_i(\theta_{i,k}) - g_i(\theta'_{i,k}))\Vert + \alpha_{i,k}\Vert g_i(\theta_t) - g_i(\theta'_t) \Vert  \nonumber   \\
        && + \alpha_{i,k}\Vert g(\theta_t) - g(\theta'_t)\Vert  \nonumber   \\
        &\le& (1 + \beta \alpha_{i,k})\Vert \theta_{i,k} - \theta'_{i,k} \Vert + 2\alpha_{i,k}\beta \Vert \theta_t - \theta'_t \Vert  .  \nonumber  
    \end{eqnarray}

    In the second case, the perturbed sample is selected at local step $k$ with probability $1/n_i$. Then,
    \begin{eqnarray}
        \Vert \theta_{i,k+1} - \theta'_{i,k+1} \Vert &\le& \Vert \theta_{i,k} - \theta'_{i,k} - \alpha_{i,k}(g_i(\theta_{i,k}) - g'_i(\theta'_{i,k}))\Vert + \alpha_{i,k}\Vert g_i(\theta_t) - g'_i(\theta'_t) \Vert  \nonumber   \\
        && + \alpha_{i,k}\Vert g(\theta_t) - g'(\theta'_t)\Vert  \nonumber   \\
        &\le& (1+\beta\alpha_{i,k})\Vert \theta_{i,k} - \theta'_{i,k} \Vert + \alpha_{i,k}\Vert g_i(\theta_t) - g_i(\theta'_t) \Vert + \alpha_{i,k}\Vert g(\theta_t) - g(\theta'_t)\Vert  \nonumber   \\
        && + \alpha_{i,k}\big( \Vert g_i(\theta'_{i,k}) - g'_i(\theta'_{i,k}) \Vert + (1+p_i)\Vert g_i(\theta'_t) - g'_i(\theta'_t) \Vert \big)  \nonumber \\
        &\le& (1 + \beta \alpha_{i,k})\Vert \theta_{i,k} - \theta'_{i,k} \Vert + 2\beta\alpha_{i,k}\Vert \theta_t - \theta'_t \Vert + \alpha_{i,k} \Vert g_i(\theta'_{i,k}) - g'_i(\theta'_{i,k}) \Vert \nonumber   \\
        && + \alpha_{i,k}(1+p_i)\Vert g_i(\theta'_t) - g'_i(\theta'_t) \Vert .   \nonumber   
    \end{eqnarray}
    Combining these two case renders
    \begin{eqnarray}
        \mathbb{E}\Vert \theta_{i,k+1} - \theta'_{i,k+1} \Vert &\le& (1 + \beta \alpha_{i,k})\mathbb{E}\Vert \theta_{i,k} - \theta'_{i,k} \Vert + 2\beta\alpha_{i,k}\mathbb{E}\Vert \theta_t - \theta'_t \Vert + \frac{2}{\alpha_{i,k}}{n_i}\mathbb{E}\Vert g_i(\theta_{i,k}) \Vert \nonumber \\
        && + \frac{2\alpha_{i,k}(1+p_i)}{n_i}\mathbb{E}\Vert g_i(\theta_t)\Vert \nonumber
    \end{eqnarray}
    and unrolling it and using Lemma \ref{lmm_grad-bnd-SCAFFOLD-nonconvex} gives
    \begin{eqnarray}    \label{eq_iter-i-SCAFFOLD-nonconvex}
        \mathbb{E}\Vert \theta_{i,K_i} - \theta'_{i,K_i} \Vert &\le& (1+2\beta\tilde{\alpha}_{i,t})e^{\beta \tilde{\alpha}_{i,t}} \mathbb{E}\Vert \theta_t - \theta'_t \Vert + \frac{2(1+p_i)}{n_i}\tilde{\alpha}_{i,t} e^{\beta \tilde{\alpha}_{i,t}} \mathbb{E}\Vert g_i(\theta_t) \Vert  \nonumber \\
        && + \frac{2\tilde{\alpha}_{i,t}}{n_i}e^{\beta \tilde{\alpha}_{i,t}} \big(\tilde{c}( \mathbb{E}\Vert \nabla R(\theta_t) \Vert + \sigma) + 2LD_i\big)
    \end{eqnarray}
    where $\tilde{c}$ is an upper bound of $1 + \beta\tilde{\alpha}_{i,t}(1 + c)^{K_i + 1}$, which is a constant and we use Lemma \ref{lmm_grad-bnd-SCAFFOLD-nonconvex}.

    Combining \eqref{eq_iter-j-SCAFFOLD-nonconvex} and \eqref{eq_iter-i-SCAFFOLD-nonconvex} we have
    \begin{eqnarray}    \label{eq_theta-iter-SCAFFOLD-nonconvex}
        \mathbb{E}\Vert \theta_{t+1} - \theta'_{t+1} \Vert &\le& \sum_{i=1}^m p_i (1+2\beta\tilde{\alpha}_{i,t}) e^{\beta \tilde{\alpha}_{i,t}} \mathbb{E}\Vert \theta_t - \theta'_t \Vert + \frac{2}{n}\sum_{i=1}^m p_i \beta\tilde{\alpha}_{i,t} e^{\beta \tilde{\alpha}_{i,t}}\mathcal{E}_t  \nonumber \\
        &&+ \frac{2}{n} \tilde{\alpha}_{i,t}e^{\beta \tilde{\alpha}_{i,t}} \mathcal{E}_t + \frac{2}{n}\tilde{\alpha}_{i,t}e^{\beta \tilde{\alpha}_{i,t}}\big(\tilde{c}(\mathbb{E}\Vert \nabla R(\theta_t) \Vert + \sigma) + 2LD_i\big) .
    \end{eqnarray}
    Finally, under the choice of stepsize stated in the theorem, unrolling \eqref{eq_theta-iter-SCAFFOLD-nonconvex} over $t$ and further using Theorem \ref{thm_stability-gen} together with the same techniques in the proof of \ref{thm_convergence}, we complete the proof.
\end{proof}

\subsection{Analysis for FedProx under non-convex losses}

\begin{lemma}   \label{lmm_drift-FedProx-nonconvex}
    Suppose Assumptions \ref{assump_Lip-continuous},\ref{assump_bounded-grad-var} hold and assume that $\nabla^2_{\theta} l(\theta;z) \succ -\mu I$ with $\mu > 0$. Considering FedProx with local update \eqref{eq_FedProx-update}, then for any $\eta_i \le \frac{1}{\mu}$, we have for any $i \in [m]$
    $$
        \mathbb{E}\Vert \theta^i_{t+1} - \theta_t \Vert \le \frac{\eta_i}{1 - \eta_i \mu} (\mathbb{E}\Vert \nabla R(\theta_t) \Vert + 2LD_i + \sigma), ~~ \forall t=0,1,\dots.
    $$
\end{lemma}
\begin{proof}
    Recalling the local update \eqref{eq_FedProx-update} of FedProx and according to the first-order optimality condition, we have
    $$
        \eta_i \nabla \hat{R}_{\mathcal{S}_i}(\theta^i_{t+1}) + \theta^i_{t+1} - \theta_t = 0.
    $$
    Moreover, since the function $\eta_i \hat{R}_{\mathcal{S}_i}(\theta) + \frac{1}{2}\Vert \theta - \theta_t \Vert$ is $1 - \eta_i \mu$-strongly-convex, we have
    $$
        \Vert \theta^i_{t+1} - \theta_t \Vert \le  \frac{\eta_i}{1 - \eta_i \mu} \Vert \nabla \hat{R}_{\mathcal{S}_i}(\theta_t) \Vert
    $$
    by combining the first-order optimality condition. Moreover, note that
    \begin{eqnarray}
        \mathbb{E}\Vert \nabla \hat{R}_{\mathcal{S}_i}(\theta_t) \Vert &\le& \mathbb{E}\Vert \nabla R_i(\theta_t) \Vert + \sigma    \nonumber   \\
        &\le& \mathbb{E}\Vert \nabla R(\theta_t) \Vert + \mathbb{E}\Vert \nabla R_i(\theta_t) - \nabla R(\theta_t) \Vert + \sigma   \nonumber   \\
        &\le& \mathbb{E}\Vert \nabla R(\theta_t) \Vert + 2LD_i + \sigma ,   \nonumber
    \end{eqnarray}
    where we use Lemma \ref{lmm_TV-bound} and note 
    $$
        \mathbb{E}\Vert \nabla \hat{R}_{\mathcal{S}_i}(\theta_t) - \nabla R_i(\theta_t) \Vert \le \frac{1}{n_i}\sum_{j=1}^{n_i} \mathbb{E}\Vert \nabla l(\theta_t;z_{i,j}) - \nabla R_i(\theta_t)\Vert \le \sigma.
    $$
    Thus, we have
    $$
        \mathbb{E}\Vert \theta^i_{t+1} - \theta_t \Vert \le \frac{\eta_i}{1-\eta_i \mu} (\mathbb{E}\Vert \nabla R(\theta_t) \Vert + 2LD_i + \sigma),
    $$
    which completes the proof.
\end{proof}

\begin{lemma}   \label{lmm_grad-bnd-FedProx-nonconvex}
    Suppose the assumptions stated in Lemma \ref{lmm_drift-FedProx-nonconvex} hold and consider FedProx with local update \eqref{eq_FedProx-update}. Then, for any $i \in [m]$ and $j \in [n_i]$, we have
    $$
        \mathbb{E}\Vert \nabla l(\theta^i_{t+1}; z_{i,j}) \Vert \le (1 + \frac{\beta \eta_i}{1 - \eta_i \mu}) (2LD_i + \mathbb{E}\Vert \nabla R(\theta_t) \Vert + \sigma), ~~ \forall t=0,1,\dots.
    $$
\end{lemma}
\begin{proof}
    For any $i \in [m]$ and time $t$,
    \begin{eqnarray}
        \mathbb{E}\Vert \nabla l(\theta^i_{t+1}; z_{i,j}) \Vert &\le& \mathbb{E}\Vert \nabla l(\theta^i_{t+1}; z_{i,j}) - \nabla R_i(\theta^i_{t+1}) \Vert + \mathbb{E}\Vert \nabla R_i(\theta^i_{t+1}) \Vert   \nonumber   \\
        &\le& \mathbb{E}\Vert \nabla R_i(\theta^i_{t+1}) \Vert + \sigma \nonumber   \\
        &\le& \mathbb{E}\Vert \nabla R_i(\theta_t) \Vert + \mathbb{E}\Vert \nabla R_i(\theta^i_{t+1}) - \nabla R_i(\theta_t) \Vert + \sigma \nonumber   \\
        &\le& \beta \mathbb{E}\Vert \theta^i_{t+1} - \theta_t \Vert + \mathbb{E}\Vert \nabla R_i(\theta_t) - \nabla R(\theta_t) \Vert + \mathbb{E}\Vert \nabla R(\theta_t) \Vert + \sigma   \nonumber   \\
        &\le& (1 + \frac{\beta \eta_i}{1 - \eta_i \mu}) (2LD_i + \mathbb{E}\Vert \nabla R(\theta_t) \Vert + \sigma),   \nonumber
    \end{eqnarray}
    where we use Lemma \ref{lmm_TV-bound} and Lemma \ref{lmm_drift-FedProx-nonconvex} in the last step.
\end{proof}

\begin{lemma}   \label{lmm_prox-nonconvex}
    Suppose $f$ is non-convex, whose eigenvalues of its Hessian are lower bounded by $-\mu$ with $0<\mu < 1$. Define the proximal operator by 
    $$
        \mathrm{prox}_{f}(x) := \arg\min_{y} f(y) + \frac{1}{2}\Vert y - x \Vert^2.
    $$
    Then, for any $x_1$, $x_2$, we have
    $$
        \Vert \mathrm{prox}_{f}(x_1) - \mathrm{prox}_{f}(x_2) \Vert \le \frac{1}{1 - \mu}\Vert x_1 - x_2 \Vert.
    $$
\end{lemma}
\begin{proof}
    Let $u_1 = \mathrm{prox}_f(x_1)$ and $u_2 = \mathrm{prox}_f(x_2)$. According to the first-order optimality condition, we have
    \begin{eqnarray}
        \nabla f(u_1) +  u_1 - x_1  &=& 0 \nonumber   \\
        \nabla f(u_2) +  u_2 - x_2  &=& 0   \nonumber
    \end{eqnarray}
    Since $\nabla^2 f$ has eigenvalues greater than $-\mu$, we further have
    \begin{eqnarray*}
        -\mu \Vert u_1 - u_2 \Vert^2 &\le& \langle \nabla f(u_1) - \nabla f(u_2), u_1 - u_2 \rangle    \\
        &=& \langle x_1 - u_2 - (x_2 - u_2), u_1 - u_2 \rangle  \\
        &=& \langle x_1 - x_2, u_1 - u_2 \rangle - \Vert u_1 - u_2 \Vert^2
    \end{eqnarray*}
    and hence
    $$
        (1-\mu)\Vert u_1 - u_2 \Vert^2 \le \langle x_1 - x_2, u_1 - u_2 \rangle \le \Vert x_1 - x_2 \Vert \Vert u_1 - u_2 \Vert
    $$
    which means
    $$
        \Vert u_1 - u_2 \Vert \le \frac{1}{1 - \mu}\Vert x_1 - x_2 \Vert .
    $$
\end{proof}

\begin{theorem}[FedProx part of Theorem \ref{thm_gen-nonconvex}]
    Suppose Assumptions \ref{assump_Lip-continuous}-\ref{assump_Lip-smooth} hold and consider FedProx (Algorithm \ref{alg_FedProx}). Assume that all eigenvalues of the Hessian of $l(\cdot;z)$ are strictly greater than $-\mu$ with $\mu > 0$ for any $z$. With $\eta_i \le \frac{\delta_t}{\mu}$ for $0<\delta < 1$ being diminishing at the order of $\mathcal{O}(c/t)$ (where $c>0$). Then, 
    \begin{align*}
        \epsilon_{gen} \le \tilde{\mathcal{O}}\Big( \frac{T^c}{n} D_{max} \Big) + \mathcal{O}\Big( \big( \Delta_0 \tilde{D} \big)^{\frac{1}{2}} \frac{T^{\frac{3}{4} + c}}{n} + \sqrt{\Delta_0} \frac{T^{\frac{1}{2} + c}}{n} \Big) + \tilde{\mathcal{O}}\Big( \frac{T^c}{n}\sigma \Big),
    \end{align*}
    where $\Delta_0 := \mathbb{E}[R(\theta_0) - R(\theta^*)]$.
\end{theorem}
\begin{proof}
    Denoting $\mathrm{prox}_f(x) := \arg\min_{y} f(y) + \frac{1}{2}\Vert y-x \Vert^2$, we can rewrite the local update \eqref{eq_FedProx-update} as
    $$
        \theta^i_{t+1} = \mathrm{prox}_{\eta_i \hat{R}_{\mathcal{S}_i}} (\theta_t).
    $$
    There are two different cases for local updates. For client $j$ with $j\ne i$, we note $\hat{R}_{\mathcal{S}_j}(\cdot) = \hat{R}_{\mathcal{S}'_j}(\cdot)$ in the sense that there is no perturbation for client $j$. In this case, using Lemma \ref{lmm_prox-nonconvex} we obtain
    \begin{eqnarray}
        \Vert \theta^i_{t+1} - (\theta^i_{t+1})' \Vert &=& \Vert \mathrm{prox}_{\eta_i \hat{R}_{\mathcal{S}_i}} (\theta_t) - \mathrm{prox}_{\eta_i \hat{R}_{\mathcal{S}_i}} (\theta'_t) \Vert  \nonumber    \\
        &\le& \frac{1}{1 - \eta_i \mu}\Vert \theta_t - \theta'_t \Vert  \nonumber  \\
        &\le& \frac{1}{1 - \delta}\Vert \theta_t - \theta'_t \Vert  \nonumber
    \end{eqnarray}

    For client $i$, we note that $\hat{R}_i(\cdot) - \hat{R}'_i(\cdot) = \frac{1}{n_i}( l(\cdot;z_{i,j}) - l(\cdot;z'_{i,j}))$, where $z'_{i,j}$ is the perturbed data point. And we also use $\hat{R}_i$ and $\hat{R}'_i$ to represent $\hat{R}_{\mathcal{S}_i}$ and $\hat{R}_{\mathcal{S}'_i}$ for simplicity. Then, we have
\begin{eqnarray}
    \theta_{t+1}^i &=& \arg\min_{\theta} \eta_i \hat{R}_i(\theta) + \frac{1}{2}\Vert \theta - \theta_t \Vert^2  \nonumber   \\
    (\theta_{t+1}^i)' &=& \arg\min_{\theta} \eta_i \hat{R}'_i(\theta) + \frac{1}{2}\Vert \theta - \theta'_t \Vert^2 . \nonumber
\end{eqnarray}
According to the first-order optimality condition, it yields
\begin{eqnarray}
    \theta_{t+1}^i - \theta_t &=& -\eta_i \nabla \hat{R}_i(\theta_{t+1}^i)  \nonumber   \\
    (\theta_{t+1}^i)' - \theta'_t &=& -\eta_i \nabla \hat{R}'_i((\theta_{t+1}^i)')  \nonumber   \\
    &=& -\eta_i \nabla \hat{R}_i((\theta_{t+1}^i)') + \frac{\eta_i}{n_i}\left( \nabla l((\theta_{t+1}^i)';z_{i,j}) - \nabla l((\theta_{t+1}^i)';z'_{i,j}) \right) . \nonumber
\end{eqnarray}
Moreover, by the techniques used in Lemma \ref{lmm_prox-nonconvex},
\begin{eqnarray}
    \Vert (\theta_{t+1}^i)' - \theta_{t+1}^i \Vert^2 &\le& \langle \theta'_t - \theta_t, (\theta_{t+1}^i)' - \theta_{t+1}^i \rangle - \eta_i \langle \nabla \hat{R}_i((\theta_{t+1}^i)') - \nabla \hat{R}_i(\theta_{t+1}^i), (\theta_{t+1}^i)' - \theta_{t+1}^i \rangle     \nonumber   \\
    && + \frac{\eta_i}{n_i} \langle (\theta_{t+1}^i)' - \theta_{t+1}^i, \nabla l((\theta_{t+1}^i)'; z_{i,j}) - \nabla l(\theta_{t+1}^i; z'_{i,j}) \rangle   \nonumber   \\
    &\le& \langle \theta'_t - \theta_t, (\theta_{t+1}^i)' - \theta_{t+1}^i \rangle + \frac{\eta_i}{n_i} \langle (\theta_{t+1}^i)' - \theta_{t+1}^i, \nabla l((\theta_{t+1}^i)'; z_{i,j}) - \nabla l(\theta_{t+1}^i; z'_{i,j}) \rangle   \nonumber  \\
    && + \eta_i \mu \Vert (\theta_{t+1}^i)' - \theta_{t+1}^i \Vert
\end{eqnarray}
which further implies by symmetry of $z_{i,j}$ and $z'_{i,j}$,
\begin{eqnarray}
    \Vert (\theta_{t+1}^i)' - \theta_{t+1}^i \Vert &\le& \frac{1}{1 - \eta_i \mu}\Vert \theta'_t - \theta_t \Vert + \frac{\eta_i}{n_i} \Vert \nabla l(\theta_{t+1}^i; z_{i,j}) - \nabla l(\theta_{t+1}^i; z'_{i,j}) \Vert    \nonumber  \\
    &\le& \frac{1}{1 - \delta}\Vert \theta'_t - \theta_t \Vert + \frac{\eta_i}{n_i} \Vert \nabla l(\theta_{t+1}^i; z_{i,j}) - \nabla l(\theta_{t+1}^i; z'_{i,j}) \Vert  \nonumber
\end{eqnarray}
where we also use Cauchy-Schwatz inequality.

Combining two cases gives
\begin{eqnarray}
    \mathbb{E}\Vert \theta_{t+1} - \theta'_{t+1} \Vert &\le& \sum_{j=1}^m p_j \mathbb{E}\Vert \theta^j_{t+1} - (\theta^j_{t+1})' \Vert  \nonumber  \\
    &\le& \frac{1}{1 - \delta}\mathbb{E}\Vert \theta_t - \theta'_t \Vert + \frac{\eta_i}{n} \mathbb{E}\Vert \nabla l(\theta_{t+1}^i; z_{i,j}) - \nabla l(\theta_{t+1}^i; z'_{i,j}) \Vert , \nonumber  \\
    &\le& \frac{1}{1 - \delta}\mathbb{E}\Vert \theta_t - \theta'_t \Vert + \frac{2\eta_i}{n} \mathbb{E}\Vert \nabla l(\theta_{t+1}^i; z_{i,j})\Vert \nonumber   \\
    &\le& \frac{1}{1 - \delta}\mathbb{E}\Vert \theta_t - \theta'_t \Vert + \frac{2\eta_i}{n} (1 + \frac{\beta \delta}{(1-\delta)\mu}) (2LD_i + \mathbb{E}\Vert \nabla R(\theta_t) \Vert + \sigma) , \nonumber
\end{eqnarray}
where we use Lemma \ref{lmm_grad-bnd-FedProx-nonconvex} in the last step. Define $\tau := \frac{\delta}{1 - \delta}$. Then,
unrolling it over $t$ we obtain
\begin{eqnarray}    \label{eq_theta-iter-FedProx-nonconvex}
    \mathbb{E}\Vert \theta_T - \theta'_T \Vert \le T^c \frac{2}{n}\sum_{t=0}^{T-1}\eta_i (1 + \beta\tau/\mu)(2LD_i + \mathbb{E}\Vert \nabla R(\theta_t) \Vert + \sigma) .
\end{eqnarray}
Finally, based on \eqref{eq_theta-iter-FedProx-nonconvex}, combining Theorem \ref{thm_stability-gen} and using the proof techniques in Theorem \ref{thm_convergence}, we complete the proof.
\end{proof}

\section{Code of the experiments}
The implementation of the experiments in Section \ref{sec_experiments} are based on \cite{platform} and can be found through the following link: \href{https://github.com/fedcodexx/Generalization-of-Federated-Learning}{https://github.com/fedcodexx/Generalization-of-Federated-Learning}.

\end{document}